\documentclass{article}

\usepackage[preprint]{neurips_2026}

\usepackage[utf8]{inputenc} 
\usepackage[T1]{fontenc}    
\usepackage{hyperref}       
\usepackage{url}            
\usepackage{booktabs}       
\usepackage{amsfonts}       
\usepackage{nicefrac}       
\usepackage{microtype}      
\usepackage[table]{xcolor}
\usepackage[most]{tcolorbox}
\usepackage{xcolor}
\usepackage{multirow}
\usepackage{multicol}
\usepackage{amssymb}
\usepackage{enumitem}
\usepackage{amsmath}
\usepackage{amsthm}
\usepackage{graphicx}
\usepackage{bbding}
\usepackage{pifont}
\usepackage{makecell}

\usepackage{float}
\usepackage{algorithm}
\usepackage{algpseudocode}

\newtheorem{theorem}{Theorem}[section]
\newtheorem{proposition}[theorem]{Proposition}

\newtheorem{assumption}[theorem]{Assumption}

\definecolor{stringcolor}{RGB}{163, 21, 21}

\newtcblisting{promptbox}[1][]{
    enhanced,
    title={#1},
    colframe=black!60,
    colback=gray!5,
    colbacktitle=black!60,
    coltitle=white,
    fonttitle=\bfseries\large,
    boxrule=3pt,
    arc=1mm,
    breakable,
    top=3mm, bottom=3mm, left=3mm, right=3mm,
    listing only,
    listing options={
        basicstyle=\small\ttfamily,
        breaklines=true,
        breakatwhitespace=true,
        columns=fullflexible,
        keepspaces=true,
        frame=none,
        aboveskip=0pt,
        belowskip=0pt,
        breakindent=0pt,
        xleftmargin=0pt,
        postbreak=\mbox{\textcolor{gray}{$\hookrightarrow$}\space},
    },
    #1
}

\newtcblisting{promptboxx}[1][]{
    enhanced,
    title={#1},
    colframe=black!60,
    colback=gray!5,
    colbacktitle=black!60,
    coltitle=white,
    fonttitle=\bfseries\large,
    boxrule=3pt,
    arc=1mm,
    breakable,
    top=3mm, bottom=3mm, left=3mm, right=3mm,
    listing only,
    listing options={
        basicstyle=\small\ttfamily,
        breaklines=true,
        breakatwhitespace=true,
        columns=fullflexible,
        keepspaces=true,
        frame=none,
        aboveskip=0pt,
        belowskip=0pt,
        breakindent=0pt,
        xleftmargin=0pt,
        postbreak=\mbox{\textcolor{gray}{$\hookrightarrow$}\space},
    },
    #1
}

\newtcblisting{promptboxxx}[1][]{
    enhanced,
    title={#1},
    colframe=black!60,
    colback=gray!5,
    colbacktitle=black!60,
    coltitle=white,
    fonttitle=\bfseries\large,
    boxrule=3pt,
    arc=1mm,
    breakable,
    top=3mm, bottom=3mm, left=3mm, right=3mm,
    listing only,
    listing options={
        basicstyle=\small\ttfamily,
        breaklines=true,
        breakatwhitespace=true,
        columns=fullflexible,
        keepspaces=true,
        frame=none,
        aboveskip=0pt,
        belowskip=0pt,
        breakindent=0pt,
        xleftmargin=0pt,
        postbreak=\mbox{\textcolor{gray}{$\hookrightarrow$}\space},
        moredelim={[is][\bfseries]{|}{|}},
        escapeinside={(*}{*)}
    },
    #1
}

\title{UPA: Unsupervised Prompt Agent via Tree-Based Search and Selection}

\author{%
  \textbf{Siran Peng}\textsuperscript{1,2,$*$}\quad
  \textbf{Weisong Zhao}\textsuperscript{4,$*$}\quad
  \textbf{Tianyu Fu}\textsuperscript{3,$*$}\quad
  \textbf{Chenxu Zhao}\textsuperscript{3}\quad
  \textbf{Tianshuo Zhang}\textsuperscript{2,1} \\
  \textbf{Haoyuan Zhang}\textsuperscript{2,1}\quad
  \textbf{Xiangyu Zhu}\textsuperscript{1,2}\quad
  \textbf{Minghui Wu}\textsuperscript{3,\dag}\quad
  \textbf{Zhen Lei}\textsuperscript{1,2,5,\dag} \\
  \textnormal{\footnotesize \textsuperscript{1}MAIS, CASIA \quad \textsuperscript{2}SAI, UCAS \quad \textsuperscript{3}Mininglamp Technology \quad \textsuperscript{4}IIE, CAS \quad \textsuperscript{5}SCSE, FIE, M.U.S.T} \\
  \texttt{\footnotesize \textnormal{\{pengsiran2023, zhen.lei\}@ia.ac.cn, zhaoweisong@iie.ac.cn, \{futianyu, wuminghui\}@mininglamp.com}} \\
  \textnormal{\footnotesize \textsuperscript{$*$}Equal contribution. \quad \textsuperscript{\dag}Corresponding author.}
}

\begin{document}

\maketitle

\begin{abstract}
Prompt agents have recently emerged as a promising paradigm for automated prompt optimization, framing prompt discovery as a sequential decision-making problem over a structured prompt space. While this formulation enables the use of advanced planning algorithms, these methods typically assume access to supervised reward signals, which are often unavailable in practical scenarios. In this work, we propose UPA, an Unsupervised Prompt Agent that realizes structured search and selection without relying on ground-truth (GT) rewards. Specifically, during search, UPA iteratively constructs an evolving tree structure to navigate the prompt space, guided by fine-grained and position-debiased pairwise comparisons from Large Language Models (LLMs). Crucially, as these local comparisons do not inherently yield a consistent global scale, we decouple systematic prompt exploration from final selection, introducing a two-stage framework grounded in the Bradley-Terry-Luce (BTL) model. This framework first performs path-wise Bayesian aggregation of local comparisons to filter candidates under uncertainty, followed by global tournament-style comparisons to infer latent prompt quality and identify the optimal prompt. Experiments across multiple tasks demonstrate that UPA consistently outperforms existing prompt optimization methods, showing that agent-style optimization can remain highly effective even in unsupervised settings.
\end{abstract}

\section{Introduction}
Well-designed prompts are critical for maximizing the capabilities of Large Language Models (LLMs) \cite{deng2023rephrase,zhengtake}. However, crafting effective prompts often requires substantial trial-and-error and deep task-specific expertise. To alleviate this burden, researchers have developed a wide range of automated prompt optimization methods that leverage LLMs themselves to iteratively improve prompt quality \cite{pryzant2023automatic,NEURIPS2023_a0054803,yang2023large,zhou2024fairer,chen2024prompt,fernando2024promptbreeder,he2025crispo,yuksekgonul2025optimizing,yan2025efficient}. As depicted in Figure~\ref{hp}, most existing approaches perform unstructured searches. Specifically, they navigate the prompt space using either single-trajectory updates \cite{yang2023large,yuksekgonul2025optimizing,yan2025efficient} or flat exploration strategies based on candidate pools or populations \cite{pryzant2023automatic,chen2024prompt,fernando2024promptbreeder}. In contrast, prompt agents have recently emerged as a promising paradigm that frames the optimization process as a sequential decision-making problem over a structured prompt space \cite{wangpromptagent,yu2025optimizing}. By maintaining explicit tree topologies and exploring diverse search trajectories, prompt agents enable systematic planning and backtracking, ultimately facilitating more robust and adaptive prompt discovery.

\begin{figure}[t]
	\begin{center}
		\begin{minipage}{1\linewidth}
			{\includegraphics[width=1\linewidth]{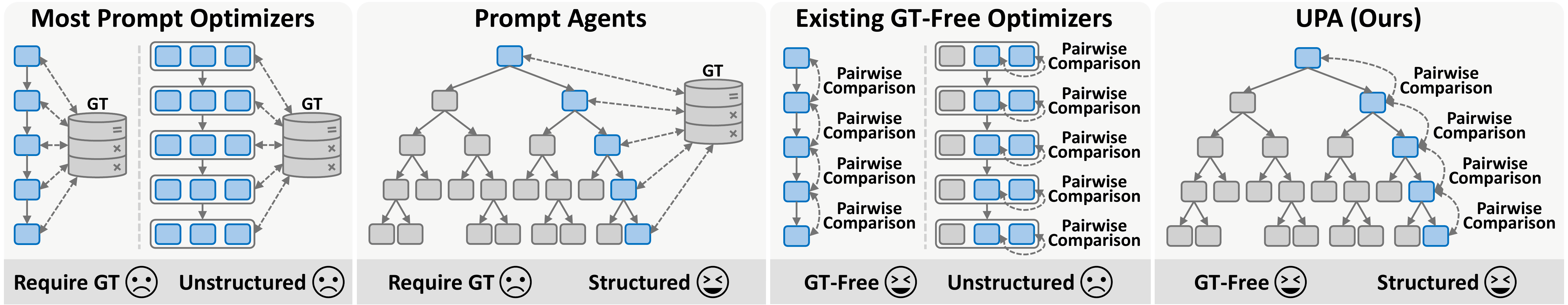}}
			\centering
		\end{minipage}
	\end{center}
    \caption{Comparison of prompt optimization paradigms, including most prompt optimizers, existing prompt agents \cite{wangpromptagent,yu2025optimizing}, existing GT-free optimizers \cite{xiang2025self,nair2025tournament,wu2025llm}, and the proposed UPA. Unlike prior methods, UPA uniquely achieves structured exploration in a fully GT-free setting. \label{hp}}
\end{figure}

Despite the success of automated prompt optimization, most existing methods rely heavily on ground-truth (GT) reward signals, which are typically derived from labeled data or task-specific metrics. Such signals are often unavailable in practical scenarios, particularly for open-ended or user-specific tasks where prompt quality is inherently difficult to quantify. Recent GT-free optimizers address this limitation by replacing absolute rewards with LLM-based pairwise comparisons: SPO \cite{xiang2025self} performs reference-free single-path refinement, while DEEVO \cite{nair2025tournament} and PDO \cite{wu2025llm} select and mutate prompts within flat candidate sets. While these methods demonstrate that relative preferences can effectively guide unstructured prompt search, it remains unclear whether such locally defined feedback can sustain multi-path exploration and cross-branch selection in tree-structured prompt agents.

Motivated by these observations, we ask a fundamental question: \textit{Can prompt agents be realized in fully GT-free settings?} To address this question, we propose UPA, an Unsupervised Prompt Agent that enables structured search and selection over the prompt space without relying on GT rewards. Specifically, during the search phase, UPA iteratively constructs an evolving tree structure to navigate the candidate space, where each node denotes a candidate prompt and each edge corresponds to a refinement step performed by an optimization LLM. This structure allows UPA to preserve explicit parent-child refinement paths, enabling path-aware exploration beyond unstructured search. Instead of computing absolute rewards, UPA utilizes a judge LLM to conduct fine-grained and position-debiased pairwise comparisons, providing relative preferences between a child prompt and its parent based on their performance on sampled inputs. Crucially, since these local comparisons do not inherently yield a consistent global ranking, UPA decouples prompt exploration from final selection. This necessitates a two-stage selection procedure grounded in the Bradley-Terry-Luce (BTL) model \cite{bradley1952rank,luce1959individual}. In the first stage, UPA performs path-wise Bayesian aggregation of noisy and sparse comparison outcomes to obtain uncertainty-aware estimates of prompt quality, filtering the tree to a promising subset of candidates. In the second stage, UPA conducts global tournament-style comparisons among these candidates, estimating their latent quality scores via BTL-based maximum likelihood estimation to identify the optimal prompt. In conclusion, the main contributions of this paper are as follows:
\begin{itemize}[leftmargin=12pt]
\item We propose UPA, an Unsupervised Prompt Agent that realizes structured search and selection over the prompt space without relying on GT rewards. This addresses a critical limitation in existing agent-style prompt optimization, which strictly requires absolute rewards to guide exploration.
\item We develop a tree-based search procedure that systematically explores the candidate space via fine-grained and position-debiased pairwise comparisons provided by a judge LLM. This is tightly coupled with a novel two-stage selection framework grounded in the BTL model, which enables robust and accurate prompt identification under noisy and sparse relative feedback conditions.
\item Through comprehensive experiments across diverse closed-ended and open-ended tasks, we demonstrate that UPA consistently outperforms state-of-the-art (SOTA) prompt optimizers, showing that agent-style prompt optimization can remain highly effective even in fully GT-free settings.
\end{itemize}

\section{Methodology}
\label{method}
\subsection{Problem Setting \& Overview}
Let $\mathcal{P}$ denote the discrete space of natural language prompts and $\mathcal{Q}$ a distribution over task queries. Given a prompt $p\in\mathcal{P}$ and a query $q\sim\mathcal{Q}$, an execution LLM produces a response $a = f_{\text{exec}}(p, q)$. The goal of prompt optimization is to identify an optimal prompt $p^\star\in\mathcal{P}$ that maximizes the expected response quality. In this work, we consider an unsupervised setting where no GT reward function is available. Instead, supervision is provided exclusively through noisy pairwise comparisons between responses. As a result, existing prompt agents that rely on supervised rewards are not applicable.

To address this challenge, we propose \textbf{UPA}, an unsupervised prompt agent guided by fine-grained and position-debiased pairwise comparisons from a judge LLM, $f_{\text{judge}}$ (Sec.~\ref{sec:pairwise}). As shown in Figure~\ref{pipeline}, UPA decouples prompt optimization into two phases: \textbf{search} and \textbf{selection}. During the search phase, UPA explores the prompt space via a tree-based strategy to generate candidates and accumulate sparse pairwise comparison data (Sec.~\ref{sec:search}). In the selection phase, to identify the optimal prompt $p^\star$ from these noisy observations, UPA employs a two-stage framework grounded in the BTL model (Sec.~\ref{sec:btl}). Specifically, path-wise Bayesian filtering is first applied to prune the search tree to the top-$K$ candidates (Sec.~\ref{sec:stage1}), followed by global BTL maximization to infer the final prompt (Sec.~\ref{sec:stage2}). The complete procedures for the search and selection phases are summarized in Algorithms~\ref{alg:upa-search} and~\ref{alg:upa-select}.

\begin{figure}[t]
	\begin{center}
		\begin{minipage}{1\linewidth}
			{\includegraphics[width=0.88\linewidth]{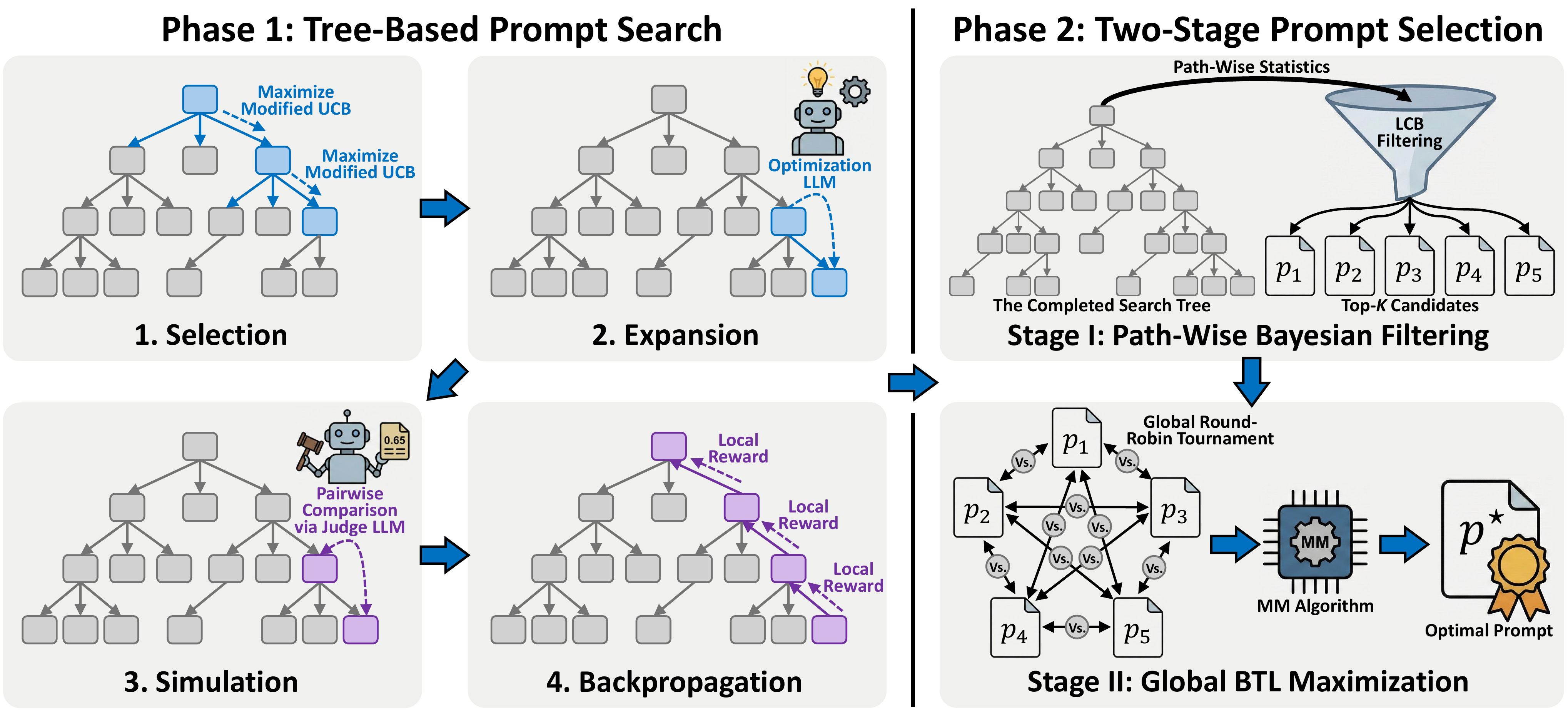}}
			\centering
		\end{minipage}
	\end{center}
	\caption{Overview of UPA. Our method decouples prompt optimization into two phases: search and selection. During search, UPA explores the prompt space using a tree-based framework driven by relative feedback. In the selection phase, it employs a two-stage strategy: path-wise Bayesian filtering to prune the search tree, followed by global BTL maximization to infer the optimal prompt. \label{pipeline}}
\end{figure}

\subsection{Fine-Grained and Position-Debiased Pairwise Comparison}
\label{sec:pairwise}
Given two prompts $p_i, p_j \in \mathcal{P}$ and a query $q$, the execution LLM generates responses $a_i=f_{\text{exec}}(p_i,q)$ and $a_j=f_{\text{exec}}(p_j,q)$. The judge LLM $f_{\text{judge}}$ then evaluates their relative quality as follows: 

\textbf{Five-Point Likert Scale.}
We quantify pairwise preferences using a discrete Likert scale $\mathcal{S} = \{1, \dots, 5\}$, following the common practice of eliciting numerical judgments from LLM evaluators \cite{NEURIPS2023_91f18a12}. Specifically, 5 signifies a strong preference for the first response, 1 denotes a strong preference for the second, and 3 represents a tie. Compared with coarse binary decisions, this scale retains fine-grained preference intensity by distinguishing strong preferences, slight preferences, and ties.

\textbf{Positional Bias Mitigation.}
LLMs exhibit inherent positional bias, often favoring the first-presented option regardless of response quality. To mitigate this bias, we evaluate each candidate pair in both presentation orders. This yields a forward score $s_{i,j} = f_{\text{judge}}(a_i, a_j, q)\in \mathcal{S}$ and a reverse score $s_{j,i} = f_{\text{judge}}(a_j, a_i, q)\in \mathcal{S}$. Since the Likert scale is symmetric, the reverse score $s_{j,i}$ is converted to $6 - s_{j,i}$ from the perspective of $a_i$. Thus, the debiased preference score $\tilde{s}_{i,j}^{(q)}$ is computed as:
\begin{equation}
\label{eq:debias}
    \tilde{s}_{i,j}^{(q)} = \frac{1}{2} \left( f_{\text{judge}}(a_i, a_j, q) + \left[ 6 - f_{\text{judge}}(a_j, a_i, q) \right] \right).
\end{equation}

\textbf{Normalization and Fractional-Response Modeling.}
To aggregate fine-grained preferences across queries, we first linearly rescale the debiased score to a continuous soft-win signal on the unit interval:
\begin{equation}
\label{eq:normalize}
y_{i,j}^{(q)} = \frac{\tilde{s}_{i,j}^{(q)} - 1}{4} \in [0, 1].
\end{equation}
This min-max normalization follows the common practice of rescaling rating-scale responses to a shared numerical range for aggregation and comparison, while treating underlying scale levels as approximately comparable increments of preference intensity \cite{dawes2008data,haslbeck2025comparing}. After normalization, we treat $y_{i,j}^{(q)}$ as a bounded fractional preference response, consistent with fractional-response modeling for variables in $[0,1]$ \cite{papke1996econometric}. Under this interpretation, each comparison contributes $y_{i,j}^{(q)}$ effective wins and $1-y_{i,j}^{(q)}$ effective losses. Over $n_{i,j}$ trials, the cumulative effective wins $w_{i,j}$ and losses $l_{i,j}$ are:
\begin{equation}
\label{eq:wins}
    w_{i,j} = \sum_{k=1}^{n_{i,j}} y_{i,j}^{(q_k)}, \quad l_{i,j} = n_{i,j} - w_{i,j}.
\end{equation}
These effective counts preserve the judge's fine-grained preference intensity and serve as fractional observations for the Beta-prior updates used in the Bayesian filtering procedure in Sec.~\ref{sec:stage1}.

\begin{figure}[t]
    \begin{minipage}[t]{0.49\textwidth}
        \begin{algorithm}[H]
        \footnotesize
            \caption{Tree-Based Prompt Search}
            \label{alg:upa-search}
            \begin{algorithmic}[1]
            \Require Initial prompt $p_o$, query pool $\mathcal{Q}$, budget $T$
            \Ensure Search tree $\mathcal{T}$
            \State Initialize root $o$ with $p_o$, $N(o) \leftarrow 0$, $W(o) \leftarrow 0$
            \For{$t = 1$ to $T$}
                \State \textcolor{blue}{\textsc{// 1. Selection}}
                \State $v \leftarrow o$
                \While{$|\text{children}(v)| = M$}
                    \State $v \leftarrow \arg\max_{k \in \text{children}(v)} \mathrm{UCB}(k,v)$
                \EndWhile

                \State \textcolor{blue}{\textsc{// 2. Expansion}}
                \If{$|\text{children}(v)| < M$}
                    \State Sample $\mathcal{Q}_{\text{exp}} \subset \mathcal{Q}$, execute to get outputs $\mathcal{A}_{\text{exp}}$
                    \State $p_{v'} \leftarrow f_{\text{opt}}(p_v, \mathcal{Q}_{\text{exp}}, \mathcal{A}_{\text{exp}})$
                    \State Instantiate $v'$ with $p_{v'}$ as a child node of $v$
                \EndIf
                
                \State \textcolor{blue}{\textsc{// 3. Simulation}}
                \State Sample $\mathcal{Q}_{\text{sim}} \subset \mathcal{Q}$
                \State Compute soft-wins $y_{v',v}^{(q)}$ via Eq.~(\ref{eq:normalize}) for all $q \in \mathcal{Q}_{\text{sim}}$
                \State Aggregate: $w_{v',v} \leftarrow \sum_{q} y_{v',v}^{(q)}, \quad n_{v',v} \leftarrow |\mathcal{Q}_{\text{sim}}|$
                \State Compute local reward: $R(v') \leftarrow w_{v',v} / n_{v',v}$
                
                \State \textcolor{blue}{\textsc{// 4. Backpropagation}}
                \State $k \leftarrow v'$
                \While{$k \neq \text{Null}$}
                    \State $N(k) \leftarrow N(k) + 1$
                    \State $W(k) \leftarrow W(k) + R(v')$
                    \State $Q(k) \leftarrow W(k) / N(k)$
                    \State $k \leftarrow \text{parent}(k)$
                \EndWhile
            \EndFor
            \State \Return $\mathcal{T}$
            \end{algorithmic}
        \end{algorithm}
    \end{minipage}
    \hfill
    \begin{minipage}[t]{0.49\textwidth}
        \begin{algorithm}[H]
        \footnotesize
            \caption{Two-Stage Prompt Selection}
            \label{alg:upa-select}
            \begin{algorithmic}[1]
            \Require Search tree $\mathcal{T}$, selection set $\mathcal{Q}_{\text{sel}}\subset \mathcal{Q}$
            \Ensure Final optimized prompt $p^\star$
            \State \textcolor{blue}{\textsc{// Stage I: Path-wise Bayesian Filtering}}
            \ForAll{node $v \in \mathcal{T}$}
                \State Initialize $\mu_v \leftarrow 0$, $\sigma_v^2 \leftarrow 0$
                \ForAll{edge $(u, k)$ in path $\mathcal{P}_{o \to v}$}
                    \State $\alpha \leftarrow \alpha_0 + w_{k,u}, \quad \beta \leftarrow \beta_0 + n_{k,u} - w_{k,u}$
                    \State $\mu_{k,u}^{\Delta} \leftarrow \psi(\alpha) - \psi(\beta)$
                    \State $(\sigma_{k,u}^{\Delta})^2 \leftarrow \psi_1(\alpha) + \psi_1(\beta)$
                    \State $\mu_v \leftarrow \mu_v + \mu_{k,u}^{\Delta}, \quad \sigma_v^2 \leftarrow \sigma_v^2 + (\sigma_{k,u}^{\Delta})^2$
                \EndFor
                \State $\mathrm{LCB}(v) \leftarrow \mu_v - \lambda_{\text{unc}} \sqrt{\sigma_v^2}$
            \EndFor
            \State $\mathcal{C} \leftarrow \text{Top-}K \text{ nodes based on } \mathrm{LCB}(v)$
            
            \State \textcolor{blue}{\textsc{// Stage II: Global BTL Maximization}}
            \State Initialize $W \in \mathbb{R}^{K \times K}$ with zeros
            \ForAll{distinct pair $p_i, p_j \in \mathcal{C}$}
                \State Set total trials: $N_{i,j} \leftarrow |\mathcal{Q}_{\text{sel}}|$
                \State Compute effective wins $W_{i,j}$ on $\mathcal{Q}_{\text{sel}}$ via Eq.~(\ref{eq:wins})
                \State Derive opponent wins: $W_{j,i} \leftarrow N_{i,j} - W_{i,j}$
            \EndFor
            \State Initialize $\boldsymbol{\gamma}^{(0)} \leftarrow \mathbf{1}$
            \Repeat \Comment{MM Algorithm}
                \ForAll{$i \in \{1, \dots, K\}$}
                    \State $\gamma_i^{(t+1)} \leftarrow \frac{\sum_{j \neq i} W_{i,j}}{\sum_{j \neq i} N_{i,j}/(\gamma_i^{(t)} + \gamma_j^{(t)})}$
                \EndFor
            \Until{Convergence}
            \State $p^\star \leftarrow \arg\max_{p_i \in \mathcal{C}} \gamma_i$
            \State \Return $p^\star$
            \end{algorithmic}
        \end{algorithm}
    \end{minipage}
\end{figure}

\subsection{Tree-Based Search with Relative Feedback}
\label{sec:search}
UPA maintains a search tree $\mathcal{T}$, where each node $v$ corresponds to a prompt $p_v$ and each edge represents a refinement operation. Beginning with an initial prompt $p_o$, the tree is expanded iteratively to explore the vast and discrete prompt space. To steer this exploration, we employ a modified Monte Carlo Tree Search (MCTS) framework \cite{coulom2006efficient} tailored to the absence of GT reward signals. Instead of relying on absolute rewards, the search is driven by relative feedback obtained through pairwise comparisons between prompts. Each search iteration proceeds through the following steps:

\textbf{Selection.} 
Starting from the root $o$, the algorithm follows a root-to-node selection path to identify a node for expansion. Along this path, if a node $v$ is not fully expanded (i.e., the number of its children is less than the maximum branching factor $M$), the selection process terminates, and $v$ is chosen. Otherwise, if $v$ is fully expanded, the algorithm selects the child node that maximizes a modified Upper Confidence Bound (UCB) score and continues along the path. This strategy ensures a balance between deepening the search in promising branches and widening the search at nodes that have not yet been fully expanded. For a node $v$ with parent $u$, its modified UCB score is defined as follows:
\begin{equation}
\label{eq:ucb}
    \mathrm{UCB}(v,u) = Q(v) + c_{\text{puct}} \sqrt{\frac{\ln N(u)}{N(v)}} - \lambda_{\text{div}} \cdot D(v),
\end{equation}
where $Q(v)$ is the empirical mean of pairwise comparison rewards backpropagated to $v$ following descendant node expansions, $N(\cdot)$ denotes the visit count, and $c_{\text{puct}}$ controls the exploration-exploitation trade-off. To discourage repeated exploration of semantically similar branches, we introduce a diversity penalty term, $D(v)$. Let $\phi(\cdot)$ denote a pre-trained sentence embedding function. We define:
\begin{equation}
    D(v) = \max\left\{0,\frac{1}{|\mathrm{sibs}(v)|} \sum_{s \in \mathrm{sibs}(v)} \cos\big(\phi(p_v), \phi(p_s)\big)\right\},
\end{equation}
where $\mathrm{sibs}(\cdot)$ denotes sibling nodes, $\cos(\cdot,\cdot)$ denotes cosine similarity, and $\lambda_{\text{div}}$ controls the penalty.

\textbf{Expansion.} 
Upon reaching the selected node $v$, we perform a single expansion step to add a new child. We sample a query batch $\mathcal{Q}_{\text{exp}}$ and obtain the execution outputs $\mathcal{A}_{\text{exp}}$ conditioned on $p_v$. These input-output pairs are provided to an optimization LLM, $f_{\text{opt}}$, to generate a single new refinement:
\begin{equation}
    p_{v'} \leftarrow f_{\text{opt}}(p_v, \mathcal{Q}_{\text{exp}}, \mathcal{A}_{\text{exp}}).
\end{equation}
Subsequently, this generated prompt $p_{v'}$ is instantiated as a child node $v'$ of $v$.

\textbf{Simulation.}  
In standard MCTS, the value of a newly expanded node is typically estimated via rollouts to a terminal state. In GT-free prompt optimization, however, extending rollouts over multiple refinement steps offers limited additional information. Since evaluation relies exclusively on local parent-child comparisons, deeper rollouts do not introduce new supervision beyond these local relations. Instead, they consume additional evaluation budget and exacerbate estimation uncertainty under finite samples. We therefore replace traditional rollouts with a fixed-budget local relative evaluation. Specifically, the newly instantiated child node $v'$ is compared against its parent node $v$ on a batch of queries $\mathcal{Q}_{\text{sim}}$. For each query $q \in \mathcal{Q}_{\text{sim}}$, we conduct a pairwise comparison following the protocol in Sec.~\ref{sec:pairwise} to obtain a soft-win signal $y_{v',v}^{(q)}$. The local reward is then computed as follows:
\begin{equation}
    R(v') = \frac{1}{|\mathcal{Q}_{\text{sim}}|} \sum_{q \in \mathcal{Q}_{\text{sim}}} y_{v',v}^{(q)}.
\end{equation}
In addition to the local reward $R(v')$, we record the effective wins $w_{v',v}$ and trials $n_{v',v}$ on the edge $(v, v')$. These quantities serve as the sufficient statistics for path-wise Bayesian filtering in Sec.~\ref{sec:stage1}.

\textbf{Backpropagation.} 
The local reward $R(v')$ is propagated upward along the entire path from $v'$ to the root $o$ to update the corresponding search statistics. For each node $k$ on this path, we update:
\begin{equation}
    N(k) \leftarrow N(k) + 1, \quad W(k) \leftarrow W(k) + R(v').
\end{equation}
The node value is computed as the empirical mean $Q(k) = W(k) / N(k)$. This backpropagation mechanism aggregates local relative improvement signals into node-level statistics, guiding future selection toward refinement paths that have consistently yielded favorable pairwise outcomes.

\subsection{Latent Quality Modeling via BTL}
\label{sec:btl}
To compare prompts discovered across different branches of the search tree, we adopt the Bradley-Terry-Luce (BTL) model to map local relative feedback onto a common latent scale. Each prompt $p_v$ is associated with a latent quality parameter $\theta_v \in \mathbb{R}$, which captures its underlying performance across queries. Let $\pi_{v,u}$ denote the soft-win probability of prompt $p_v$ over prompt $p_u$, i.e., the expected fractional preference response induced by the pairwise comparison protocol in Sec.~\ref{sec:pairwise}. Under the BTL model, this probability is given by the sigmoid of their latent quality difference:
\begin{equation}
    \pi_{v,u} = \frac{e^{\theta_v}}{e^{\theta_v} + e^{\theta_u}} = \sigma(\theta_v - \theta_u),
\end{equation}
where $\sigma(x) = 1/(1+e^{-x})$ is the sigmoid function. Inverting this relationship, the difference in latent quality corresponds to the log-odds (logit) of the pairwise soft-win rate:
\begin{equation}
    \theta_v - \theta_u = \ln \frac{\pi_{v,u}}{1 - \pi_{v,u}} \triangleq \mathrm{logit}(\pi_{v,u}).
\end{equation}
This logit parameterization converts each local soft-win probability into an additive quality increment. Since the tree topology guarantees a unique path $\mathcal{P}_{o \to v}$ from the root $o$ to any node $v$, these increments can be telescoped along the path. Thus, the root-relative latent quality of a node decomposes as:
\begin{equation}
\label{eq:telescope}
    \theta_v - \theta_o = \sum_{(u,k) \in \mathcal{P}_{o \to v}} \underbrace{(\theta_k - \theta_u)}_{\Delta \theta_{k,u}} = \sum_{(u,k) \in \mathcal{P}_{o \to v}} \mathrm{logit}(\pi_{k,u}),
\end{equation}
where each term $(u, k)$ in the summation represents a directed refinement edge from parent $u$ to child $k$, and $\Delta \theta_{k,u}$ denotes the local quality increment associated with that refinement step.

\subsection{Stage I Selection: Path-Wise Bayesian Filtering}
\label{sec:stage1}
While Eq.~(\ref{eq:telescope}) decomposes root-relative latent quality into local increments, directly summing empirical estimates can be statistically unstable. During search, each edge $(u, v)$ is evaluated using a fixed and limited sampling budget to control computational cost, which results in high-variance estimates of the local increment $\Delta \theta_{v,u}$. When such noisy estimates are accumulated along a path, uncertainty compounds with depth, leading to poorly calibrated quality estimates for deeper nodes.

To address this, we propose a path-wise Bayesian filtering algorithm. For each edge $(u,v)$, we place a Beta prior on the unknown soft-win probability $\pi_{v,u}$. Following the fractional-response interpretation in Sec.~\ref{sec:pairwise}, the effective wins and losses $\mathcal{D}_{v,u}=\{w_{v,u},l_{v,u}\}$ serve as fractional pseudo-counts. This remains conjugate because the fractional likelihood is proportional to $\pi_{v,u}^{w_{v,u}}(1-\pi_{v,u})^{l_{v,u}}$, and the Beta distribution admits positive real-valued shape parameters. The posterior distribution is:
\begin{equation}
    \pi_{v,u} \mid \mathcal{D}_{v,u} \sim \mathrm{Beta}(\alpha_{v,u}, \beta_{v,u}),
\end{equation}
where $\alpha_{v,u} = \alpha_0 + w_{v,u}$ and $\beta_{v,u} = \beta_0 + l_{v,u}$. We adopt a non-informative uniform prior by setting $\alpha_0 = \beta_0 = 1$. With the posterior established, we characterize the local quality increment $\Delta \theta_{v,u} = \mathrm{logit}(\pi_{v,u})$ through its first two moments. Specifically, the mean $\mu_{v,u}^{\Delta}$ and variance $(\sigma_{v,u}^{\Delta})^2$ admit closed-form expressions involving the Digamma $\psi(\cdot)$ and Trigamma $\psi_1(\cdot)$ functions:
\begin{equation}
\label{eq13}
\begin{aligned}
    \mu_{v,u}^{\Delta} &= \mathbb{E}[\mathrm{logit}(\pi_{v,u})] = \psi(\alpha_{v,u}) - \psi(\beta_{v,u}), \\
    (\sigma_{v,u}^{\Delta})^2 &= \mathrm{Var}[\mathrm{logit}(\pi_{v,u})] = \psi_1(\alpha_{v,u}) + \psi_1(\beta_{v,u}).
\end{aligned}
\end{equation}

To aggregate these local estimates, we compute path-wise posterior statistics along $\mathcal{P}_{o \to v}$. By linearity of expectation, the path-wise posterior mean $\mu_v$ is the sum of local posterior means. To quantify total uncertainty, we require an assumption about correlations among edge-level estimation errors.
\begin{assumption}
\label{ass:ind}
Conditioned on the edge-level soft-win probabilities, the estimation errors of the local logit increments along a refinement path are statistically independent.
\end{assumption}
We discuss the rationale, potential violations, and empirical diagnostic of this assumption in Appendix~\ref{app:assumption_discussion}. Under Assumption~\ref{ass:ind}, the variances of the local increments along a refinement path become additive. Thus, the path-wise posterior mean $\mu_v$ and total variance $\sigma_v^2$ are as follows:
\begin{equation}
\label{eq14}
    \mu_v = \sum_{(u,k) \in \mathcal{P}_{o \to v}} \mu_{k,u}^{\Delta}, \quad \sigma_v^2 = \sum_{(u,k) \in \mathcal{P}_{o \to v}} (\sigma_{k,u}^{\Delta})^2.
\end{equation}
To conservatively identify promising prompts while accounting for accumulated uncertainty, we rank candidate nodes using an uncertainty-penalized Lower Confidence Bound (LCB) score, defined as:
\begin{equation}\label{eq15}
    \mathrm{LCB}(v) = \mu_v - \lambda_{\text{unc}} \sqrt{\sigma_v^2},
\end{equation}
where $\lambda_{\text{unc}}$ is a hyperparameter controlling risk aversion. This formulation penalizes nodes that have high posterior means but also high accumulated uncertainty. Based on this ranking, we retain the top-$K$ candidates. The following proposition characterizes the asymptotic behavior of this filtering:
\begin{proposition}
\label{proposition2.1}
Under the fractional-response BTL model, the path-wise posterior mean $\mu_v$ provides a consistent estimate of the latent quality difference $\theta_v - \theta_o$. Moreover, under Assumption~\ref{ass:ind}, for any finite path, as the sampling budget on each edge increases, the path-wise variance $\sigma_v^2$ vanishes. Consequently, the uncertainty penalty in $\mathrm{LCB}(v)$ vanishes asymptotically, and the induced ranking concentrates around the BTL latent quality ordering when the quality gaps are nonzero.
\end{proposition}

\subsection{Stage II Selection: Global BTL Maximization}
\label{sec:stage2}
While Stage~I efficiently prunes the search space via path-wise inference, the resulting quality estimates are inherently path-dependent and not directly comparable across candidates from different refinement trajectories. To identify the optimal prompt $p^\star$ using path-independent evidence, we conduct a global round-robin tournament over the top-$K$ candidates, denoted as $\mathcal{C} = \{p_1, \dots, p_K\}$.

To ensure robust quality estimation for the final selection, we introduce a separately sampled selection set $\mathcal{Q}_{\text{sel}}$, where $|\mathcal{Q}_{\text{sel}}| > |\mathcal{Q}_{\text{sim}}|$. For every pair of candidates $p_i, p_j \in \mathcal{C}$, we perform pairwise comparisons on $\mathcal{Q}_{\text{sel}}$ following the protocol in Sec.~\ref{sec:pairwise}. Let $N_{i,j} = |\mathcal{Q}_{\text{sel}}|$ denote the number of effective comparison trials between $p_i$ and $p_j$, and let $W_{i,j}$ represent the cumulative effective wins of $p_i$ over $p_j$. By construction, the effective wins of $p_j$ satisfy $W_{j,i} = N_{i,j} - W_{i,j}$. We estimate the latent quality parameters $\boldsymbol{\theta} = [\theta_1, \dots, \theta_K]^\top$ by maximizing the fractional BTL log-likelihood:
\begin{equation}
    \mathcal{L}(\boldsymbol{\theta}) = \sum_{i=1}^K \sum_{j \neq i} W_{i,j} \ln \frac{e^{\theta_i}}{e^{\theta_i} + e^{\theta_j}}.
\end{equation}
This optimization problem is solved using the standard Minorization-Maximization (MM) algorithm \cite{hunter2004mm}. Reparameterizing the quality score as $\gamma_i = e^{\theta_i}$, the update rule at iteration $t$ is:
\begin{equation}
    \gamma_i^{(t+1)} = \frac{\sum_{j \neq i} W_{i,j}}{\sum_{j \neq i} \frac{N_{i,j}}{\gamma_i^{(t)} + \gamma_j^{(t)}}}.
\end{equation}
After convergence, we select the prompt with the highest estimated quality: $p^\star = \arg\max_{p_i \in \mathcal{C}} \hat{\theta}_i$.

\begin{table}[t]	
	\centering\renewcommand\arraystretch{1.03}\setlength{\tabcolsep}{3.3pt}
	\belowrulesep=0pt\aboverulesep=0pt
        \caption{Performance comparison on closed-ended benchmarks. The method marked with $\ast$ is reproduced by us, while other baseline results are sourced from \cite{xiang2025self}. The best and second-best scores are highlighted in \textbf{bold} and \underline{underlined}. All methods use GPT-4o-mini \cite{hurst2024gpt} as the execution LLM and are evaluated on the same test sets, with results averaged over three independent runs.\label{closed-ended}}
	\begin{tabular}{c|c|ccccc|c}
		\toprule
		\multirow{2}{*}{Method} & 
		\multirow{2}{*}{Venue} & 
		\multicolumn{5}{c|}{Accuracy (\%)}&
        \multirow{2}{*}{Avg.}\\
		\cmidrule(lr){3-7}
		&\multicolumn{1}{c|}{} 
		&\multicolumn{1}{c}{GPQA} 
		&\multicolumn{1}{c}{AGIEval-MATH} 
		&\multicolumn{1}{c}{LIAR}
		&\multicolumn{1}{c}{WSC} 
        &\multicolumn{1}{c|}{BBH-Navigate}\\
		\midrule
		  IO 
          & - & 38.9 & 42.1 & 63.5 & 72.4 & 91.3 & 61.6 \\
		  CoT \cite{NEURIPS2022_9d560961}
          & NeurIPS 22 & 41.6 & 44.5 & 65.4 & 77.8 & 89.7 & 63.8 \\
		  RaR \cite{deng2023rephrase}
          & arXiv 23 & 40.2 & 42.1 & 50.5 & 79.1 & 93.5 & 61.1 \\
	\midrule
        APE \cite{zhou2022large} 
          & ICLR 23 & 41.1 & 44.4 & 65.9 & 80.2 & 92.5 & 64.8 \\
        OPRO \cite{yang2023large}
          & ICLR 24 & 43.3 & 46.1 & 67.6 & 80.2 & 95.8 & 66.6 \\
        PromptAgent \cite{wangpromptagent}
          & ICLR 24 & 41.3 & 41.4 & 64.1 & \textbf{82.7} & 95.7 & 65.0 \\
        PromptBreeder \cite{fernando2024promptbreeder} 
          & ICML 24 & 40.9 & 45.9 & 63.2 & 76.7 & \underline{96.3} & 64.6 \\
        TextGrad \cite{yuksekgonul2025optimizing}
          & Nature 25 & 40.2 & 44.4 & 65.7 & 78.0 & 91.3 & 63.9 \\
        \midrule
		SPO \cite{xiang2025self}
          & EMNLP 25 & 41.8 & 45.3 & 66.9 & 81.1 & \underline{96.3} & 66.3\\
        PDO\textsuperscript{$\ast$} \cite{wu2025llm}
          & ACL 26 & \underline{43.6} & \underline{47.7} & \underline{67.9} & \underline{82.0} & 96.2 & \underline{67.5} \\
         \rowcolor{gray!20}
		UPA (Ours) 
          & - & \textbf{45.5} & \textbf{52.1} & \textbf{68.2} & \textbf{82.7} & \textbf{98.0} & \textbf{69.3} \\      
		\bottomrule
	\end{tabular}
\end{table}

\section{Experiments}
\label{experiment}
\subsection{Setup}
\label{setup}
\textbf{Datasets.}
We evaluate UPA on both closed-ended and open-ended tasks. For closed-ended assessment, we use five widely adopted benchmarks: GPQA \cite{rein2024gpqa}, AGIEval-MATH \cite{zhong2024agieval}, LIAR \cite{wang2021new}, WSC \cite{levesque2012winograd}, and BBH-Navigate \cite{suzgun2023challenging}. These datasets provide concise and verifiable reference answers, enabling accuracy-based evaluation. To ensure fair and controlled comparisons with prior works, we follow the same data partitioning and subset selection protocols as \cite{xiang2025self}. For open-ended assessment, we select the Writing, Roleplay, and Humanities sub-tasks from MT-Bench \cite{NEURIPS2023_91f18a12}, which emphasize instruction adherence, coherence, and generative quality. For MT-Bench, we retain the original benchmark questions exclusively for evaluation and construct 10 auxiliary questions per sub-task, disjoint from the test set and used solely for optimization. Additional dataset details are provided in Appendix~\ref{dataset_details}.

\textbf{Baselines.}
For closed-ended tasks, we compare UPA with three categories of baselines. The first includes conventional prompting strategies: Input-Output (IO), Chain-of-Thought (CoT) \cite{NEURIPS2022_9d560961}, and Rephrase-and-Respond (RaR) \cite{deng2023rephrase}. The second covers GT-dependent prompt optimization methods: APE \cite{zhou2022large}, OPRO \cite{yang2023large}, PromptAgent \cite{wangpromptagent}, PromptBreeder \cite{fernando2024promptbreeder}, and TextGrad \cite{yuksekgonul2025optimizing}. The third includes GT-free prompt optimizers: SPO \cite{xiang2025self} and PDO \cite{wu2025llm}. For open-ended tasks, we compare UPA with conventional prompting strategies and GT-free optimizers, including IO, CoT, RaR, SPO, and PDO.

\textbf{Evaluation Metrics.}
We employ standard accuracy as the evaluation metric for all closed-ended benchmarks. For the open-ended MT-Bench, we adopt the LLM-as-a-judge framework to perform pairwise comparisons following \cite{xiang2025self}. Results are reported as UPA's pairwise win rate against each baseline, where a win rate exceeding 50\% signifies an overall preference for our approach.

\begin{table}[t]
	\centering\renewcommand\arraystretch{1.05}\setlength{\tabcolsep}{3pt}
	\belowrulesep=0pt\aboverulesep=0pt
	\caption{Cross-model generalization results for GT-free prompt optimizers on frontier execution LLMs. Full average-accuracy comparisons with all baselines are provided in Appendix~\ref{extend_generalization}. \label{tab:cross_model}}
	\begin{tabular}{c|c|ccccc|c}
		\toprule
		\multirow{2}{*}{Execution LLM} & \multirow{2}{*}{Method} & \multicolumn{5}{c|}{Accuracy (\%)} & \multirow{2}{*}{Avg.} \\
		\cmidrule(lr){3-7}
		& & GPQA & AGIEval-MATH & LIAR & WSC & BBH-Navigate & \\
		\midrule
		\multirow{3}{*}{GPT-5} 
		& SPO \cite{xiang2025self} & 82.5 & 94.9 & 78.0 & 98.2 & 98.5 & 90.4 \\
        & PDO \cite{wu2025llm} & 83.3 & 95.2 & 78.3 & 97.6 & 99.0 & 90.7 \\
		& \cellcolor{gray!20} UPA (Ours) & \cellcolor{gray!20}\textbf{84.2} & \cellcolor{gray!20}\textbf{95.7} & \cellcolor{gray!20}\textbf{78.8} & \cellcolor{gray!20}\textbf{98.5} & \cellcolor{gray!20}\textbf{100.0} & \cellcolor{gray!20}\textbf{91.4} \\
		\midrule
		\multirow{3}{*}{Claude-4.5-Sonnet} 
		& SPO \cite{xiang2025self} & 73.7 & 84.7 & 74.1 & 97.3 & 99.8 & 85.9 \\
        & PDO \cite{wu2025llm} & 74.7 & 85.8 & 74.2 & 97.1 & 99.5 & 86.3 \\
		& \cellcolor{gray!20} UPA (Ours) & \cellcolor{gray!20}\textbf{75.8} & \cellcolor{gray!20}\textbf{86.6} & \cellcolor{gray!20}\textbf{74.7} & \cellcolor{gray!20}\textbf{98.0} & \cellcolor{gray!20}\textbf{100.0} & \cellcolor{gray!20}\textbf{87.0} \\
		\midrule
		\multirow{3}{*}{DeepSeek-V3.2} 
		& SPO \cite{xiang2025self} & 73.6 & 86.3 & 68.2 & 93.1 & 99.3 & 84.1 \\
        & PDO \cite{wu2025llm} & 75.1 & 87.8 & 68.3 & 93.3 & 99.3 & 84.8 \\
		& \cellcolor{gray!20} UPA (Ours) & \cellcolor{gray!20}\textbf{78.3} & \cellcolor{gray!20}\textbf{93.1} & \cellcolor{gray!20}\textbf{68.6} & \cellcolor{gray!20}\textbf{93.8} & \cellcolor{gray!20}\textbf{99.7} & \cellcolor{gray!20}\textbf{86.7} \\
		\bottomrule
	\end{tabular}
\end{table}

\textbf{Implementation Details.}
Consistent with \cite{xiang2025self}, UPA uses GPT-4o-mini \cite{hurst2024gpt} as both the execution LLM ($f_\text{exec}$, temperature $0.0$) and judge LLM ($f_\text{judge}$, temperature $0.3$), and GPT-4o \cite{hurst2024gpt} as the optimization LLM ($f_\text{opt}$, temperature $0.7$). All methods use the same execution template, \texttt{\{p\}\{q\}\{req\}}: the query \texttt{\{q\}} and formatting requirement \texttt{\{req\}} are fixed across methods, while only the prompt \texttt{\{p\}} varies. Starting from a CoT prompt $p_o$, UPA runs for $T=100$ search iterations with $c_\text{puct}=0.1$ to balance exploration and exploitation. To reduce semantic redundancy, we use a diversity penalty $\lambda_\text{div}=0.2$ with text-embedding-3-small \cite{openai2024embeddings} as the embedding function $\phi(\cdot)$. We set the branching factor to $M=3$, expansion batch size to $|\mathcal{Q}_{\text{exp}}|=5$, and simulation batch size to $|\mathcal{Q}_{\text{sim}}|=5$. In the selection phase, we use an uncertainty coefficient $\lambda_\text{unc}=0.5$ and retain the top $K=5$ candidates for a global round-robin tournament on a selection set of size $|\mathcal{Q}_{\text{sel}}|=10$. Finally, the MM algorithm follows standard convergence criteria, with a maximum of $100$ iterations and a threshold of $10^{-4}$.

\begin{table}[t]
\centering
\small 
\renewcommand\arraystretch{1.1}\setlength{\tabcolsep}{1.7pt}
\belowrulesep=0pt\aboverulesep=0pt
\caption{Pairwise win rates (\%) of UPA against conventional prompting strategies and GT-free optimizers on open-ended MT-Bench sub-tasks. GT-dependent optimizers are excluded because MT-Bench lacks GT labels. A win rate $>50\%$ indicates that UPA is preferred by the judge LLM.\label{tab:open_ended}}
\begin{tabular}{l|ccc|ccc|ccc}
\toprule
\multirow{2}{*}{Comparison} & \multicolumn{3}{c|}{GPT-5} & \multicolumn{3}{c|}{Claude-4.5-Sonnet} & \multicolumn{3}{c}{DeepSeek-V3.2} \\
\cmidrule(lr){2-4} \cmidrule(lr){5-7} \cmidrule(lr){8-10}
 & Writing & Roleplay & Humanities & Writing & Roleplay & Humanities & Writing & Roleplay & Humanities \\
\midrule
UPA vs. IO  & 55.0 & 62.5 & 68.3 & 60.0 & 62.5 & 85.0 & 71.7 & 70.0 & 75.0 \\
UPA vs. CoT & 52.5 & 64.2 & 61.7 & 56.7 & 55.8 & 64.2 & 55.8 & 54.2 & 63.3 \\
UPA vs. RaR & 53.3 & 68.3 & 65.0 & 57.5 & 70.8 & 88.3 & 66.7 & 73.3 & 70.0 \\
\midrule
UPA vs. SPO & 50.8 & 66.7 & 57.5 & 53.3 & 50.8 & 58.3 & 53.3 & 55.8 & 70.0 \\
UPA vs. PDO & 51.7 & 56.7 & 58.3 & 51.7 & 61.7 & 55.0 & 52.5 & 53.3 & 66.7 \\
\bottomrule
\end{tabular}
\end{table}

\begin{figure}[t]
	\begin{center}
		\begin{minipage}{1\linewidth}
			{\includegraphics[width=0.93\linewidth]{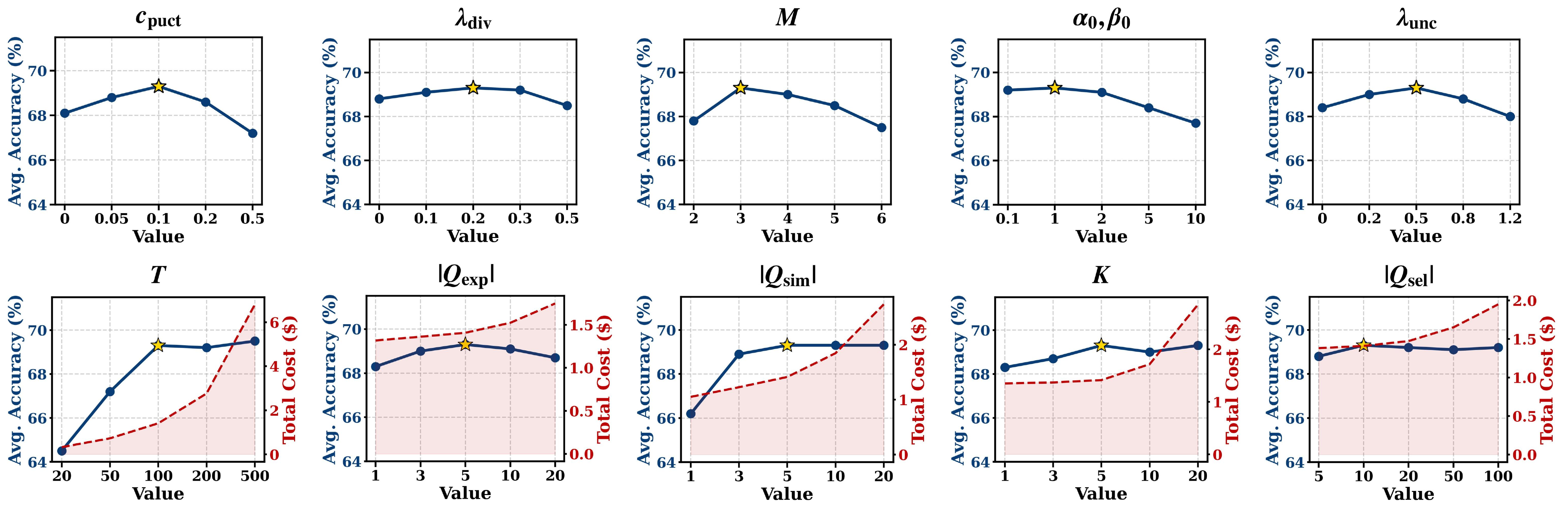}}
			\centering
		\end{minipage}
	\end{center}
	\caption{Sensitivity analysis of hyperparameters. The top and bottom rows show strategy-oriented and scale-oriented parameters, respectively. Our default settings balance performance and efficiency.\label{hyper}}
\end{figure}

\subsection{Performance}
\textbf{Closed-ended Tasks.}
Results in Table~\ref{closed-ended} show that UPA consistently outperforms recent SOTA prompt optimization baselines. To evaluate the robustness of our method, we further conduct experiments across three frontier execution LLMs: GPT-5 \cite{openai_gpt5_2025}, Claude-4.5-Sonnet \cite{anthropic_claude45_2025}, and DeepSeek-V3.2 \cite{liu2025deepseek}. Notably, the prompts are optimized under the GPT-4o-mini/GPT-4o configuration and directly applied to these target models without re-optimization. As shown in Table~\ref{tab:cross_model}, UPA maintains the best performance across executors, demonstrating strong cross-model generalizability. \textbf{Cost analysis}, \textbf{optimized prompts}, and \textbf{search tree visualizations} are provided in Appendices~\ref{cost_analysis},~\ref{optimized_prompts}, and~\ref{search_tree}.

\textbf{Open-ended Tasks.}
We evaluate UPA on open-ended tasks using the same frontier executors: GPT-5, Claude-4.5-Sonnet, and DeepSeek-V3.2. As shown in Table~\ref{tab:open_ended}, UPA outperforms all conventional prompting strategies and GT-free optimizers. These results demonstrate the practical effectiveness of our method. The \textbf{optimized prompts} for MT-Bench sub-tasks are provided in Appendix~\ref{optimized_prompts}.

\subsection{Ablation Studies}
\textbf{Comparison Protocols \& Search Strategies.}
As shown in Table~\ref{tab:search}, the ablation study results validate the effectiveness of our comparison protocols and search strategies. For comparison protocols, replacing the Five-Point Likert Scale with binary scores or removing positional debiasing consistently reduces accuracy, highlighting the need for fine-grained and unbiased feedback. For search strategies, removing the diversity penalty or exploration term in the modified UCB score also degrades performance. Notably, replacing parent-child relative comparisons with static comparisons against the initial root prompt causes a substantial performance drop ($-3.4\%$). Without local relative feedback, the search tree tends to degenerate into a breadth-first structure: since most refined candidates easily outperform the naive root prompt, the search becomes dominated by the exploration term.

\textbf{Selection Strategies.}
To validate our selection framework, we systematically ablate the LCB ranking (using posterior means only), Bayesian modeling (replacing posterior inference with empirical averages), and Stage II selection. As summarized in Table~\ref{tab:selection}, removing any component leads to a notable performance decline. Furthermore, we compare UPA against two standard baselines: choosing the candidate with the maximum $Q(v)$ or the highest local win rate. These alternatives cause significant degradation ($-4.2\%$ and $-4.6\%$), demonstrating that raw statistics from the search phase are insufficient for robust identification. This underscores the necessity of our two-stage selection in mitigating sampling noise and path-dependent estimates through mathematical design.

\begin{table}[t]
    \centering
    \begin{minipage}[t]{0.49\textwidth}
        \centering
        \caption{Ablation study on comparison protocols and search strategies. Our method performs best.}
        \label{tab:search}
        \scriptsize
        \setlength{\tabcolsep}{2pt}        
        \renewcommand\arraystretch{1.2}
        \belowrulesep=0pt\aboverulesep=0pt
        \begin{tabular}{c|ccccc|c}
        \toprule
        \multirow{2}{*}{Method} & \multicolumn{5}{c|}{Accuracy (\%)} & \multirow{2}{*}{Avg.} \\
        \cmidrule(lr){2-6}
        & GPQA & MATH & LIAR & WSC & BBH & \\
        \midrule
        w/o Likert Scale (Binary) & 43.7 & 51.8 & 67.0 & 82.4 & 97.3 & 68.4 \\
        w/o Positional Debiasing  & 44.6 & 52.0 & 67.9 & 81.2 & 96.8 & 68.5 \\
        w/o Diversity Penalty  & 44.8 & \textbf{52.8} & 67.5 & 81.5 & 97.5 & 68.8 \\
        w/o Exploration in UCB & 44.9 & 49.9 & \textbf{68.4} & 80.2 & 97.0 & 68.1 \\
        w/o Local Comparison  & 42.6 & 49.3 & 65.9 & 77.3 & 94.3 & 65.9 \\
        \rowcolor{gray!20}
        The Proposed & \textbf{45.5} & 52.1 & 68.2 & \textbf{82.7} & \textbf{98.0} & \textbf{69.3} \\
        \bottomrule
        \end{tabular}
    \end{minipage}
    \hfill
    \begin{minipage}[t]{0.49\textwidth}
        \centering
        \caption{Ablation study on selection strategies. Our two-stage selection outperforms all others.}
        \label{tab:selection}
        \scriptsize
        \setlength{\tabcolsep}{2pt}
        \renewcommand\arraystretch{1.2}
        \belowrulesep=0pt\aboverulesep=0pt
        \begin{tabular}{c|ccccc|c}
        \toprule
        \multirow{2}{*}{Method} & \multicolumn{5}{c|}{Accuracy (\%)} & \multirow{2}{*}{Avg.} \\
        \cmidrule(lr){2-6}
        & GPQA & MATH & LIAR & WSC & BBH & \\
        \midrule
        w/o LCB (Mean-only) & 44.1 & 51.2 & 68.1 & 82.5 & 96.0 & 68.4 \\
        w/o Bayesian Modeling & 45.3 & 50.8 & 67.0 & 79.8 & 97.7 & 68.1 \\
        w/o Stage II Selection  & \textbf{45.8} & 51.4 & 67.2 & 82.2 & 95.0 & 68.3 \\
        Max $Q(v)$ & 42.4 & 45.7 & 66.4 & 78.2 & 93.0 & 65.1 \\
        Max Local Win Rate  & 40.2 & 48.0 & 65.9 & 77.1 & 92.3 & 64.7 \\
        \rowcolor{gray!20}
        The Proposed & 45.5 & \textbf{52.1} & \textbf{68.2} & \textbf{82.7} & \textbf{98.0} & \textbf{69.3} \\
        \bottomrule
        \end{tabular}
    \end{minipage}
\end{table}

\textbf{Hyperparameters.}
We categorize the hyperparameters into strategy-oriented (cost-independent) and scale-oriented (cost-dependent) groups. The former, including $c_{\text{puct}}$, $\lambda_{\text{div}}$, $M$, $(\alpha_0, \beta_0)$, and $\lambda_{\text{unc}}$, governs the search and selection logic without directly changing the number of LLM calls. The latter, including $T$, $|\mathcal{Q}_{\text{exp}}|$, $|\mathcal{Q}_{\text{sim}}|$, $K$, and $|\mathcal{Q}_{\text{sel}}|$, directly scales the computational budget. As depicted in Figure~\ref{hyper}, our default configuration provides a favorable balance between performance and efficiency. We further provide an ablation study on the \textbf{judge and optimization LLMs} in Appendix~\ref{judge_and_opt}. 

\section{Conclusion}
In this paper, we propose UPA, an unsupervised prompt agent that enables structured search and selection in the absence of GT rewards. During search, UPA iteratively constructs an evolving tree structure to explore the prompt space, guided by fine-grained and position-debiased pairwise comparisons. In the selection phase, UPA employs a two-stage framework grounded in the BTL model to identify the optimal prompt from noisy and sparse comparison outcomes. Experiments across multiple tasks show that UPA consistently outperforms SOTA prompt optimization methods, proving that agent-style prompt optimization remains highly effective even in fully GT-free settings.

\bibliographystyle{unsrt}  
\bibliography{reference}


\appendix

\section{Related Works}
\subsection{Prompt Engineering}
Prompt engineering studies how human-designed instructions and inference procedures improve LLM performance. Existing works can be broadly categorized into task-agnostic reasoning frameworks and domain-specific applications. The first category focuses on task-agnostic techniques that enhance general reasoning capabilities. Beyond foundational Chain-of-Thought (CoT) prompting, which elicits stepwise intermediate outputs \cite{NEURIPS2022_9d560961,NEURIPS2022_8bb0d291,NEURIPS2023_dfc310e8,xia2025beyond}, subsequent work organizes model inference into richer computational structures. Techniques such as Tree of Thoughts (ToT) \cite{NEURIPS2023_271db992,ranaldi2024tree} and Graph of Thoughts (GoT) \cite{besta2024graph} allocate additional test-time compute to branching, comparison, and aggregation of intermediate states, supporting more systematic problem solving on complex tasks. The second category addresses domain-specific prompting tailored to distinct task types. In software engineering, prompt and agent workflows support planning, self-debugging, and modular program synthesis \cite{hong2023metagpt,ridnik2024code,adnan2025large,pantraining}. For data analysis and question answering, prompt designs emphasize structured analysis, tool use, and contextual reasoning \cite{wu2024chartinsights,hong2025data,zhang2025data,tang2025llm}. In decision-making and interactive settings, prompts specify roles, objectives, and action constraints to facilitate agentic behaviors and environment interaction \cite{wangvoyager,schmidgall2025agent,xu2025nuclear}. While these manually designed prompting frameworks, ranging from general reasoning patterns to domain-specific workflows, are effective, their design relies heavily on expert intuition and iterative trial-and-error. Moreover, these approaches primarily structure reasoning or task execution under fixed prompts, rather than systematically searching the prompt space itself. This scalability limitation motivates research in automated prompt optimization.

\subsection{Prompt Optimization}
Automated prompt optimization aims to algorithmically search for prompts that improve model performance over a task distribution. Existing methods differ along two key dimensions: the feedback signal used for selection and the topology used to organize search. GT-dependent methods typically rely on labeled validation data or task-specific metrics to guide optimization. Some methods perform sequential refinement, where an optimization LLM proposes new prompts based on historical scores or textual feedback, as in OPRO \cite{yang2023large} and TextGrad \cite{yuksekgonul2025optimizing}. Others maintain flat candidate pools or populations, generating and selecting prompts through scoring, mutation, or evolutionary operations, as in APE \cite{zhou2022large} and PromptBreeder \cite{fernando2024promptbreeder}. Beyond these unstructured searches, prompt agents formulate prompt optimization as a decision-making problem over a structured prompt space. PromptAgent \cite{wangpromptagent} and MCTS-OPS \cite{yu2025optimizing} instantiate this paradigm with MCTS-based, reward-guided planning over explicit search trees. However, these structured prompt agents still rely on GT-derived rewards, limiting their applicability in fully GT-free settings. Recent GT-free prompt optimizers replace absolute rewards with LLM-based pairwise preferences. SPO \cite{xiang2025self} performs reference-free single-path refinement, while DEEVO \cite{nair2025tournament} and PDO \cite{wu2025llm} extend GT-free optimization to flat pool-based or population-based search. These methods show that relative preferences can guide GT-free prompt search, but they either follow a single refinement trajectory or operate over flat candidate sets without an explicit structured prompt topology. In contrast, UPA studies whether GT-free feedback can support agent-style structured exploration, path-wise aggregation, and cross-branch selection.

Importantly, prompt optimization is methodologically distinct from test-time compute methods in prompt engineering, such as ToT and GoT \cite{NEURIPS2023_271db992,besta2024graph}. Test-time compute approaches allocate additional inference budget to each input, for example by sampling, branching, comparing, or aggregating multiple reasoning paths during answer generation. In contrast, prompt optimization amortizes computation during an offline optimization phase: once a prompt is optimized, it is executed as a fixed prompt on future inputs. Therefore, prompt optimization algorithms must address different challenges, including prompt-space exploration, noisy prompt-level evaluation, candidate selection under limited optimization budgets, and generalization of the optimized prompt across inputs.

\section{Additional Information on Methodology}
\label{appendix_method}
\subsection{Discussion on Assumption 2.1}
\label{app:assumption_discussion}
In this section, we discuss the rationale, potential violations, and supporting empirical diagnostics of Assumption~\ref{ass:ind}, which assumes that, conditioned on edge-level soft-win probabilities, the estimation errors of the local logit increments along a refinement path are statistically independent.

\textbf{Rationale.} 
This assumption is introduced to make path-wise uncertainty estimation feasible in a fully GT-free setting. Without independence, the variance of a path-wise sum would require covariance terms between edge-level estimation errors. However, without GT-derived quality measurements, these covariance terms are not identifiable from pairwise preference observations alone, creating a structural limitation for covariance estimation. Assumption~\ref{ass:ind} therefore provides a tractable additive approximation for path-wise variance. Methodologically, this approximation is supported by UPA's simulation design: each refinement edge is evaluated on a separately sampled query batch $\mathcal{Q}_{\text{sim}}$ from the optimization pool $\mathcal{Q}$, which reduces direct observation reuse across successive edges.

\textbf{Potential Violations.} 
We acknowledge that strict independence may be violated due to semantic dependencies between prompts. Since a child prompt is derived by refining its parent, the two prompts may share similar strengths or failure modes on certain query types. Consequently, estimation errors along a refinement path may exhibit positive correlations; for example, if a parent prompt is overestimated on a particular subset of queries, its child may inherit a similar overestimation pattern.

\textbf{Empirical Diagnostic.} 
To examine this concern, we conduct a post-hoc diagnostic over all generated prompts from the 15 search trees in our closed-ended experiments, corresponding to five closed-ended tasks across three independent runs. Using test-set accuracy increments as an external proxy for true edge-level improvement, we standardize local logit estimates and accuracy increments across edges and measure the Pearson correlation between successive estimation errors. The resulting correlation is low ($0.06$), suggesting that error dependence along refinement paths is weak in our experiments. This supports the use of the additive variance approximation for Stage~I filtering in our setting.

\subsection{Proof of Proposition 2.2}
\label{proof2.2}
\textbf{Restatement of Proposition 2.2.}
\textit{Under the fractional-response BTL model, the path-wise posterior mean $\mu_v$ provides a consistent estimate of the latent quality difference $\theta_v - \theta_o$. Moreover, under Assumption~\ref{ass:ind}, for any finite path, as the sampling budget on each edge increases, the path-wise variance $\sigma_v^2$ vanishes. Consequently, the uncertainty penalty in $\mathrm{LCB}(v)$ vanishes asymptotically, and the induced ranking concentrates around the BTL latent quality ordering when the gaps are nonzero.}

\begin{proof}
The proof has two parts: consistency of the posterior mean and vanishing variance.

\textbf{Consistency of the Path-Wise Posterior Mean.}
Recall from Eq.~(\ref{eq:telescope}) that the root-relative latent quality difference decomposes into a sum of local log-odds: $\theta_v - \theta_o = \sum_{(u,k) \in \mathcal{P}_{o \to v}} \mathrm{logit}(\pi_{k,u})$. Under the fractional-response BTL model, each $\pi_{k,u}$ denotes the soft-win probability associated with the directed refinement edge $(u,k)$. In the Bayesian filtering procedure, the local posterior mean $\mu_{k,u}^{\Delta}$ derived in Eq.~(\ref{eq13}) is the posterior expectation of the local logit increment $\Delta\theta_{k,u}=\mathrm{logit}(\pi_{k,u})$. By linearity of expectation, the path-wise posterior mean defined in Eq.~(\ref{eq14}) aggregates these local posterior means: $\mu_v = \sum_{(u,k) \in \mathcal{P}_{o \to v}} \mu_{k,u}^{\Delta}$. Although the posterior mean of the logit can exhibit finite-sample bias due to the nonlinearity of the logit transformation, this bias vanishes asymptotically as the posterior over each $\pi_{k,u}$ concentrates around its true soft-win probability. Therefore, for any finite path, $\mu_v$ converges asymptotically to $\theta_v-\theta_o$ as the edge-level sampling budgets increase.

\textbf{Vanishing Variance and Convergence.}
We next examine the path-wise variance $\sigma_v^2$. Under Assumption~\ref{ass:ind}, edge-level estimation errors along a refinement path are independent conditioned on their soft-win probabilities, so for a finite path, the total variance is additive as in Eq.~(\ref{eq14}): $\sigma_v^2 = \sum_{(u,k) \in \mathcal{P}_{o \to v}} (\sigma_{k,u}^{\Delta})^2$. The local edge-wise variance term in Eq.~(\ref{eq13}) is $(\sigma_{k,u}^{\Delta})^2 = \psi_1(\alpha_{k,u})+\psi_1(\beta_{k,u})$, where $\psi_1(\cdot)$ is the Trigamma function. Since $\psi_1(z)\sim 1/z$ for large $z$, and since $\alpha_{k,u}$ and $\beta_{k,u}$ grow linearly with the number of comparison trials $n_{k,u}$, we obtain the asymptotic limit:
\begin{equation}
\lim_{n_{k,u} \to \infty} (\sigma_{k,u}^{\Delta})^2 = \lim_{n_{k,u} \to \infty} \left( \psi_1(\alpha_{k,u}) + \psi_1(\beta_{k,u}) \right) = 0.
\end{equation}
Because the path length is finite, summing over finitely many local variances that vanish for every edge gives $\sigma_v^2\to 0$. Finally, Eq.~(\ref{eq15}) defines $\mathrm{LCB}(v)=\mu_v-\lambda_{\text{unc}}\sqrt{\sigma_v^2}$. As $\mu_v\to\theta_v-\theta_o$ and $\sigma_v^2\to 0$, the uncertainty penalty vanishes and $\mathrm{LCB}(v)$ therefore converges in probability to the latent quality difference $\theta_v-\theta_o$. Consequently, when latent quality gaps among candidate nodes are nonzero, the ranking induced by $\mathrm{LCB}$ concentrates around the BTL latent quality ordering.
\end{proof}

\subsection{Prompts for the Judge and Optimization LLMs}
We present the prompts for the judge and optimization LLMs below, using only simple instructions.

\begin{promptbox}[title=Prompt for the Judge LLM]
"""
Evaluate the two responses (A and B) based on the question and requirements.

<instruction>
Compare the two responses and determine which one better meets the requirements and accurately answers the question. 
Do NOT blindly favor the first response. Compare strictly based on content.
You must choose one of the following decisions:
- "A_MUCH_BETTER":     Response A is SIGNIFICANTLY better than B.
- "A_BETTER":          Response A is SLIGHTLY better than B.
- "TIE":               Both are equally good or equally bad.
- "B_BETTER":          Response B is SLIGHTLY better than A.
- "B_MUCH_BETTER":     Response B is SIGNIFICANTLY better than A.
</instruction>

<requirement>
{req}
</requirement>

<question>
{q}
</question>

<response_a>
{a}
</response_a>

<response_b>
{b}
</response_b>

Please strictly follow this XML format for your output:

<analyse>
Brief comparison analysis.
</analyse>

<decision>
One of [A_MUCH_BETTER, A_BETTER, TIE, B_BETTER, B_MUCH_BETTER]
</decision>
"""
\end{promptbox}

\begin{promptbox}[title=Prompt for the Optimization LLM]
"""
You are optimizing a prompt to better satisfy the given task requirement.
The goal is to improve the prompt while preserving its ability to generalize across diverse inputs.

<context>
The reference prompt has performed well in a previous iteration.
You must further refine it, making meaningful improvements.
The optimized prompt must differ from the reference prompt.
</context>

<requirement>
{req}
</requirement>

<reference_prompt>
{p}
</reference_prompt>

<execution_results>
The following are example outputs produced by the reference prompt on sampled inputs:
{ex}
</execution_results>

Output Format:
You MUST follow the XML format below.

<analyse>
Briefly analyze the limitations of the reference prompt and identify concrete ways to address them.
</analyse>

<modification>
Summarize the key improvement in one concise sentence.
</modification>

<prompt>
Write the full text of the optimized prompt here.
</prompt>
"""
\end{promptbox}

\section{Additional Information on Experiments}
\label{appendix_expeirment}
\subsection{Dataset Details}
\label{dataset_details}
\begin{table}[h]
\centering
\renewcommand\arraystretch{1.1}\setlength{\tabcolsep}{6pt}
\belowrulesep=0pt\aboverulesep=0pt
\caption{Statistics of the evaluation and optimization subsets used in our experiments. The $\ast$ indicates that MT-Bench optimization samples are auxiliary questions constructed outside the standard benchmark test set (10 per selected sub-task) and used solely for optimization purposes.}\label{tab:datasets}
\begin{tabular}{c|ccccc|c}
\toprule
\multirow{2}{*}{Partition} & \multicolumn{6}{c}{Sample Count} \\
\cmidrule(lr){2-7}
& GPQA & AGIEval-MATH & LIAR & WSC & BBH-Navigate & MT-Bench \\
\midrule
Evaluation Set & 198 & 256 & 461 & 150 & 200 & 30 \\
Optimization Set & 250 & 232 & 3681 & 50 & 50 & 30\textsuperscript{$\ast$} \\
\bottomrule
\end{tabular}
\end{table}

We evaluate UPA on a diverse set of benchmarks covering closed-ended reasoning and decision-making tasks as well as open-ended generation scenarios. Following \cite{xiang2025self}, we adopt the same data partitioning and subset selection protocols to ensure a fair and consistent comparison, with sample counts summarized in Table~\ref{tab:datasets}. Specifically, \textbf{GPQA}~\cite{rein2024gpqa} contains graduate-level multiple-choice questions across biology, physics, and chemistry, and is used to assess expert-level reasoning. \textbf{AGIEval-MATH}~\cite{zhong2024agieval} evaluates mathematical problem solving with challenging questions from standardized exams. \textbf{LIAR}~\cite{wang2021new} targets truthfulness classification, requiring models to determine the veracity of statements with contextual metadata. \textbf{WSC}~\cite{levesque2012winograd} evaluates commonsense reasoning through Winograd Schema pronoun resolution. \textbf{BBH-Navigate}~\cite{suzgun2023challenging}, from BIG-Bench Hard, assesses navigational reasoning by asking whether an agent returns to its starting point after a sequence of steps. Finally, for open-ended assessment, we use the Writing, Roleplay, and Humanities sub-tasks from \textbf{MT-Bench}~\cite{NEURIPS2023_91f18a12} to evaluate instruction following, coherence, and generative quality. For MT-Bench, we retain the original benchmark questions exclusively for evaluation and construct 10 auxiliary questions per selected sub-task, disjoint from the test set and used solely for optimization purposes.

\subsection{Conventional Prompting Baselines}
\label{conventional_prompting_baselines}
This section clarifies the conventional baselines used in our experiments. IO denotes direct input-output prompting, where the prompt \texttt{\{p\}} is empty. The CoT and RaR prompts are shown below.

\begin{promptboxx}[title=CoT Prompt]
Let's think step by step.
\end{promptboxx}

\begin{promptboxx}[title=RaR Prompt]
Please rephrase the question in a way that is easier to understand, minimizing ambiguity and considering edge cases. 
And then provide a solution step by step for the question.
\end{promptboxx}

\subsection{Optimization Cost Analysis}
\label{cost_analysis}
\begin{figure}[t]
    \begin{center}
        \begin{minipage}{1\linewidth}
            {\includegraphics[width=0.9\linewidth]{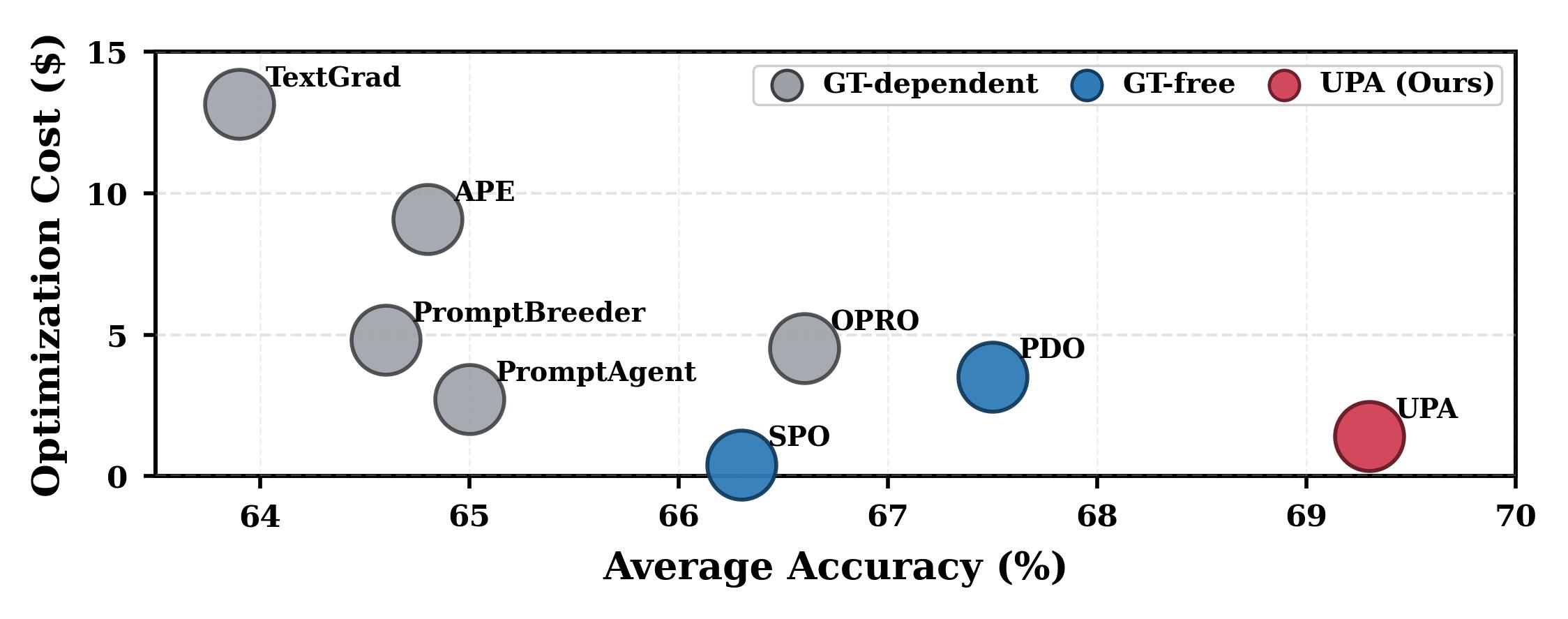}}
            \centering
        \end{minipage}
    \end{center}
    \caption{Cost-performance trade-off among prompt optimizers on closed-ended benchmarks.\label{fig:cost_performance}}
\end{figure}

\begin{table}[t]
\centering
\renewcommand\arraystretch{1.1}\setlength{\tabcolsep}{2.6pt}
\belowrulesep=0pt\aboverulesep=0pt
\caption{Detailed cost breakdown of UPA per dataset run. The majority of the budget is allocated to the search phase to ensure sufficient exploration of the prompt space. Abbreviations: Exp. (Expansion), Sim. (Simulation), Opt. (Optimization), Exec. (Execution), and Eval. (Evaluation).}
\label{tab:upa_breakdown_refined}
\begin{tabular}{c|ccccc|ccc|c}
\toprule
\multirow{2}{*}{Metric} & \multicolumn{5}{c|}{Search Phase} & \multicolumn{3}{c|}{Selection Phase} & \multirow{2}{*}{Total} \\
\cmidrule(lr){2-6} \cmidrule(lr){7-9}
 & Exp. Opt. & Exp. Exec. & Sim. Exec. & Sim. Eval. & Subtotal & Exec. & Eval. & Subtotal & \\
\midrule
Cost (\$) & 0.79 & 0.11 & 0.23 & 0.22 & 1.35 & 0.01 & 0.05 & 0.06 & 1.41 \\
Ratio (\%) & 56.0 & 7.8 & 16.3 & 15.6 & 95.7 & 0.7 & 3.6 & 4.3 & 100.0 \\
\bottomrule
\end{tabular}
\end{table}

To evaluate the economic efficiency of UPA, we conduct an optimization cost analysis using standard API pricing as of early 2026. Specifically, we use GPT-4o-mini (\$0.15/\$0.60 per 1M input/output tokens) for execution and evaluation, and GPT-4o (\$2.50/\$10.00 per 1M input/output tokens) for optimization. For baseline optimizers except PDO, we report the costs from \cite{xiang2025self}, since the unit prices of the involved models (GPT-4o-mini, GPT-4o, and Claude-3.5-Sonnet) have remained unchanged. The cost of PDO is computed from our reproduction. As shown in Figure~\ref{fig:cost_performance}, UPA achieves the best average accuracy with a relatively low optimization cost. Compared with other optimizers except SPO, UPA is substantially less expensive while delivering stronger performance. Although SPO is cheaper, it follows a single-path refinement strategy and shows substantially lower accuracy in our experiments. Moreover, when SPO is scaled to a comparable optimization budget, it reaches only 65.2\% average accuracy, well below UPA's 69.3\%, highlighting the limitation of single-path refinement. Notably, the cost differences arise only during prompt optimization, while all methods use fixed prompts at test time and therefore incur roughly the same inference cost. A detailed breakdown of UPA's cost is presented in Table~\ref{tab:upa_breakdown_refined}. The majority of the budget (95.7\%) is allocated to the search phase, with 56.0\% dedicated to expansion optimization. This design concentrates GPT-4o usage on prompt refinement, where stronger reasoning is most useful, while keeping the selection phase efficient (4.3\% of the total budget). Overall, this cost structure supports robust candidate discovery while avoiding expensive global comparisons over all generated prompts during final selection. The cost of embedding generation is omitted as it remains negligible at approximately \$0.0002 per dataset.

\subsection{Evaluation Stability and Calibration}
\label{evaluation_stability}

\begin{table}[h]
\centering
\renewcommand\arraystretch{1.08}\setlength{\tabcolsep}{7.5pt}
\belowrulesep=0pt\aboverulesep=0pt
\caption{Mean accuracy (\%) $\pm$ standard deviation over three runs under the GPT-5 executor.}
\label{tab:stability_gpt5}
\begin{tabular}{c|ccccc}
\toprule
\multirow{2}{*}{Method} & \multicolumn{5}{c}{Mean Accuracy (\%) $\pm$ Standard Deviation} \\
\cmidrule(lr){2-6}
& GPQA & AGIEval-MATH & LIAR & WSC & BBH-Navigate \\
\midrule
SPO \cite{xiang2025self} & $82.5 \pm 2.3$ & $94.9 \pm 1.4$ & $78.0 \pm 0.9$ & $98.2 \pm 0.8$ & $98.5 \pm 0.5$ \\
\rowcolor{gray!20} UPA (Ours) & $\textbf{84.2} \pm 2.1$ & $\textbf{95.7} \pm 1.4$ & $\textbf{78.8} \pm 0.8$ & $\textbf{98.5} \pm 0.4$ & $\textbf{100.0} \pm 0.0$ \\
\bottomrule
\end{tabular}
\end{table}

To further assess evaluation stability, we report the mean and standard deviation across three independent runs for SPO and UPA under the GPT-5 executor. As shown in Table~\ref{tab:stability_gpt5}, UPA outperforms SPO across all closed-ended benchmarks, while its standard deviation is no larger than SPO's on every dataset. This indicates that UPA effectively maintains stable performance across independent runs.

We also examine the calibration between the two selection stages. Across all 15 runs, corresponding to five closed-ended benchmarks and three runs, the average Spearman correlation between the top-5 candidates' Stage I LCB rankings and Stage II BTL rankings is $0.618$. This positive correlation suggests that Stage I provides an effective directional filter before the final global tournament.

\subsection{Full Cross-Model Generalization Results}
\label{extend_generalization}
\begin{table}[h]
      \centering\renewcommand\arraystretch{1.05}\setlength{\tabcolsep}{3.7pt}
      \belowrulesep=0pt\aboverulesep=0pt
      \caption{Full cross-model generalization results on frontier LLMs. We report average accuracy across all closed-ended benchmarks. PA and PB denote PromptAgent and PromptBreeder, respectively.
  \label{tab:cross_model_full}}
      \begin{tabular}{c|ccccccccccc}
          \toprule
          \multirow{2}{*}{Executor} & \multicolumn{11}{c}{Average Accuracy (\%)} \\
          \cmidrule(lr){2-12}
          & IO & CoT & RaR & APE & OPRO & PA & PB & TextGrad & SPO & PDO & \cellcolor{gray!20}UPA \\
          \midrule
          GPT-5
          & 90.0 & 90.7 & 89.6 & 90.2 & 90.8 & 91.1 & 90.5 & 90.5 & 90.4 & 90.7 & \cellcolor{gray!20}\textbf{91.4} \\
          Claude-4.5-Sonnet
          & 86.2 & 85.8 & 85.1 & 85.5 & 86.6 & 86.5 & 85.8 & 86.1 & 85.9 & 86.3 & \cellcolor{gray!20}\textbf{87.0} \\
          DeepSeek-V3.2
          & 82.9 & 85.3 & 83.4 & 84.8 & 86.3 & 85.8 & 86.1 & 85.5 & 84.1 & 84.8 & \cellcolor{gray!20}\textbf{86.7} \\
          \bottomrule
      \end{tabular}
\end{table}

Table~\ref{tab:cross_model_full} provides the full comparison with all baseline methods under frontier execution LLMs. UPA achieves the highest average accuracy on all executors, demonstrating strong cross-model robustness.

\subsection{Ablation Study on the Judge and Optimization LLMs}
\label{judge_and_opt}
\begin{table}[h]
\small
\centering
\renewcommand\arraystretch{1.1}\setlength{\tabcolsep}{1pt}
\belowrulesep=0pt\aboverulesep=0pt
\caption{Ablation study on the judge and optimization LLMs. We report the average accuracy and cost for $3 \times 3$ combinations. The results indicate that our default configuration, using GPT-4o-mini as the judge and GPT-4o for optimization, achieves a favorable performance-cost balance. Abbreviations: Opt (Optimization), Mini (GPT-4o-mini), 4o (GPT-4o), and Sonnet (Claude-4.5-Sonnet).}
\label{tab:ablation_matrix}
\begin{tabular}{c|ccc|ccc|ccc}
\toprule
\multirow{2}{*}{Metric} & \multicolumn{3}{c|}{Judge: Mini (\$0.15/\$0.6)} & \multicolumn{3}{c|}{Judge: 4o (\$2.5/\$10)} & \multicolumn{3}{c}{Judge: Sonnet (\$3/\$15)} \\
\cmidrule(lr){2-4} \cmidrule(lr){5-7} \cmidrule(lr){8-10}
 & Opt: Mini & \cellcolor{gray!20}Opt: 4o & Opt: Sonnet & Opt: Mini & Opt: 4o & Opt: Sonnet & Opt: Mini & Opt: 4o & Opt: Sonnet \\
\midrule
Avg. Acc (\%) & 66.5 & \cellcolor{gray!20}69.3 & 69.2 & 66.7 & 69.4 & 69.3 & 66.8 & 69.3 & 69.5 \\
Cost (\$) & 0.65 & \cellcolor{gray!20}1.41 & 1.71 & 4.76 & 5.51 & 5.82 & 6.02 & 6.77 & 7.08 \\
\bottomrule
\end{tabular}
\end{table}
To validate the rationale behind our default configuration (judge LLM: GPT-4o-mini, optimization LLM: GPT-4o), we conduct a comprehensive ablation study over $3 \times 3$ combinations of judge and optimization LLMs, including the frontier Claude-4.5-Sonnet. We adopt GPT-4o and GPT-4o-mini primarily to maintain consistency with prior works, ensuring a fair and controlled comparison. Beyond this protocol-level consideration, the results in Table~\ref{tab:ablation_matrix} further demonstrate that, under this standardized setup, our configuration lies on a favorable trade-off between performance and cost. For the judge LLM, upgrading from GPT-4o-mini to GPT-4o or Claude-4.5-Sonnet leads to a multi-fold increase in cost ($>3\times$) while yielding negligible performance improvements ($\leq 0.3\%$). This observation suggests that, under our fine-grained and position-debiased pairwise comparison design, even the relatively lightweight GPT-4o-mini is sufficient as a judge model. For the optimization LLM, GPT-4o emerges as the practical ``sweet spot'', substantially outperforming GPT-4o-mini (69.3\% vs.\ 66.5\%) while remaining more cost-effective than Claude-4.5-Sonnet. Notably, substituting the optimizer with the more expensive Claude-4.5-Sonnet does not improve performance (69.2\% vs.\ 69.3\%) and increases the overall cost by 21\%. We attribute this saturation effect to an execution bottleneck: since task execution is performed by GPT-4o-mini, the achievable performance is inherently constrained by the executor's capabilities, which GPT-4o is already able to fully exploit. Consequently, further increasing the optimizer's model capacity yields diminishing returns, validating our default configuration as an economically efficient and methodologically consistent design choice.

\subsection{Illustration of LLM-Based Pairwise Comparison}
\label{pairwise_comparison_example}
To illustrate LLM-based pairwise comparison, we provide an example from GPQA below. Given two candidate responses to the same question, the judge LLM compares their reasoning quality and final decisions, identifying Response A as better because it follows clearer problem-solving steps to match the closest option, whereas Response B exhibits logical inconsistency and guesses the answer.

\begin{promptbox}[title=Representative Judge Evaluation]
<question>
Arrange the following carbocations in decreasing order of stability:
1. CH3OCH2(+)
2. CH2(+)-NO2
3. CH2(+)-CHO
4. CH3COCH2(+)
5. CH2(+)-OH
6. CH3CH2(+)
7. CH2(+)CH2Cl
Options:
A. 5>6>7>1>2>3>4
B. 1>5>7>6>4>3>2
C. 5>1>6>7>4>3>2
D. 4>3>2>5>1>6>7
</question>

<response_a>
[Excerpt] ...The order of stability from most to least stable based on the analysis is: 1 > 5 > 4 > 6 > 7 > 3 > 2... Now, we will match this order with the provided options... The closest match to our derived order is option C.
Conclusion: <answer>C</answer>
</response_a>

<response_b>
[Excerpt] ...After careful analysis, the correct order of stability is not represented in the provided options. Since none of the options match the derived stability order, we conclude that the provided options may contain an error.
Final Answer: <answer>A</answer>
</response_b>

<judge_evaluation>
<analyse>
Response A provides a detailed analysis of each carbocation's stability... It clearly outlines the reasoning behind the stability order and matches it with the provided options, ultimately identifying option C as the closest match. Response B also analyzes the carbocations but concludes that none of the options accurately represent the derived stability order, which may indicate a lack of confidence in the provided choices. While both responses are informative, Response A is more structured and directly addresses the question with a clear conclusion.
</analyse>
  
<decision>
A_BETTER
</decision>
</judge_evaluation>
\end{promptbox}

\subsection{List of Optimized Prompts}
\label{optimized_prompts}
In this section, we present the optimized prompts generated by UPA utilizing the default GPT-4o-mini/GPT-4o configuration. We cover both closed-ended tasks (GPQA, AGIEval-MATH, LIAR, WSC, and BBH-Navigate) and open-ended MT-Bench sub-tasks (Writing, Roleplay, and Humanities). Notably, the resulting prompts exhibit substantial diversity in content and structure, suggesting that effective instruction strategies likely differ markedly depending on specific task requirements.

\begin{promptboxx}[title=Optimized Prompt for GPQA]
To address the problem effectively, break it down into logical, concise steps with clear explanations at each stage. Consider alternative methods or perspectives that could lead to the correct solution, and analyze how each step contributes to the final choice.
When faced with multiple potential answers, evaluate each thoroughly by verifying calculations and assumptions. Cross-reference results with given answer choices to ensure alignment with the final answer. Emphasize accuracy by double-checking intermediate steps and considering all interpretations of the problem to avoid errors.
Prioritize presenting the final answer strictly in XML format as follows: <answer>A</answer>, <answer>B</answer>, <answer>C</answer>, or <answer>D</answer>.
\end{promptboxx}

\begin{promptboxx}[title=Optimized Prompt for AGIEval-MATH]
Decompose the problem into distinct, clearly defined steps. Provide detailed mathematical explanations using LaTeX for each step, ensuring clarity and understanding.
Highlight any special conditions or edge cases encountered during the solution process, and validate the final result for precision.
Format the final answer consistently in XML LaTeX as specified, ensuring accuracy.
\end{promptboxx}

\begin{promptboxx}[title=Optimized Prompt for LIAR]
To improve the accuracy of fact-checking a statement, adhere to these refined steps: 
1. Gather evidence from a diverse array of credible sources such as official documents, statistical data, and expert analyses, and ensure to cite these sources appropriately. Focus on recent and primary sources that directly relate to the statement. 
2. Critically assess the consistency and reliability of the collected evidence, considering any conflicting information and alternative interpretations with an objective lens. 
3. Contextualize the statement to differentiate between subjective opinions and verifiable factual claims. 
4. Address any ambiguous or contradictory evidence explicitly, providing a thorough explanation of why it complicates the conclusion, and be prepared to reassess the evidence if new, credible information arises. 
Deliver a concise rationale for your conclusion, detailing how the evidence supports or refutes the statement. Finally, present the statement's truthfulness in the specified XML format, using <answer>True</answer> or <answer>False</answer>.
\end{promptboxx}

\begin{promptboxx}[title=Optimized Prompt for WSC]
Examine the sentence by identifying the main subjects and actions, and determine the logical relationships and context that clarify the pronoun in question. 
Consider the consequences of each action and the typical roles of the entities involved, ensuring that each step addresses potential ambiguities and how the overall context informs the pronoun's reference.
\end{promptboxx}

\begin{promptboxx}[title=Optimized Prompt for BBH-Navigate]
Analyze each instruction by systematically mapping direction changes, steps taken, and calculating net displacement from the starting point.
Clarify directional conventions, such as `left' and `right' relative to the current orientation, to maintain accurate tracking.
Ensure all movement sequences are tracked in a clear, structured format to accurately assess if the endpoint coincides with the starting point.
\end{promptboxx}

\begin{promptboxx}[title=Optimized Prompt for the Writing Sub-task in MT-Bench]
Craft a response that is not only clear, well-structured, and captivating but also finely attuned to the user's specific context and needs.
Infuse the response with creativity and emotional engagement, ensuring it resonates deeply with the audience. Address the key elements of the question with sufficient detail, employing vivid language to create a strong emotional connection.
Strive for a balance between thoroughness and readability, making the response both informative and engaging while ensuring it aligns with the user's unique situation or inquiry.
\end{promptboxx}

\begin{promptboxx}[title=Optimized Prompt for the Roleplay Sub-task in MT-Bench]
Respond directly to the user's request with detailed, context-aware answers that creatively address the nuances of the query.
Incorporate diverse stylistic elements and imaginative storytelling to enhance engagement and enrich the user experience.
\end{promptboxx}

\begin{promptboxx}[title=Optimized Prompt for the Humanities Sub-task in MT-Bench]
When responding to the user's request, deliver an answer that is both direct and comprehensive, integrating a diverse array of perspectives such as cultural, philosophical, and interdisciplinary insights. Ensure your explanations are clear and supported by relevant examples, including contemporary cases and emerging trends where applicable.
Tailor the depth and complexity to suit the audience's understanding, explicitly balancing detail with accessibility to maintain engagement without overwhelming the reader.
Maintain a logical flow and structure for enhanced readability and coherence, while also incorporating relevant historical, contemporary, and other contexts to provide a balanced analysis that effectively informs and engages the audience.
\end{promptboxx}

\subsection{Sample Outputs on MT-Bench}
In this section, we present sample outputs from MT-Bench, executed on GPT-5. We compare responses generated using standard IO prompting against those produced with UPA-optimized prompts (see Appendix~\ref{optimized_prompts} for the full text of the prompts). As illustrated in the following examples, the responses derived from UPA exhibit superior quality, coherence, and adherence to instructions compared to the IO baseline, qualitatively demonstrating the effectiveness of our method.

\begin{promptboxxx}[title=Sample Outputs on the Writing Sub-task]
(*\textbf{Instruction:}*)
Draft a professional email seeking your supervisor's feedback on the `Quarterly Financial Report' you prepared. Ask specifically about the data analysis, presentation style, and the clarity of conclusions drawn. Keep the email short and to the point.

(*\textbf{Response (IO):}*)
Subject: Feedback Request: Quarterly Financial Report

Hi [Supervisor's Name],

I've attached the draft of the Quarterly Financial Report and would value your feedback on:
- Data analysis: approach, accuracy, and key insights
- Presentation style: structure, visuals, and readability
- Clarity of conclusions: whether they're well-supported and actionable
If convenient, please share comments or a time for a brief 15-minute review this week. I can incorporate revisions by [date].

Thank you,
[Your Name]

(*\textbf{Response (UPA):}*)
Subject: Requesting your quick feedback on the Quarterly Financial Report

Hi [Supervisor Name],

I've attached the Quarterly Financial Report and would value your quick feedback before I share it wider. I want to ensure the narrative is sharp and the numbers tell a clear story. Could you please weigh in on:
- Data analysis: Are the methods, assumptions, and variance drivers sound and complete?
- Presentation style: Do the visuals and layout make the key points easy to grasp?
- Conclusions: Are the takeaways and recommended actions clear, balanced, and defensible?
I can turn revisions around today and am happy to walk through highlights if helpful. Thanks so much for your guidance.

Best regards,
[Your Name]
\end{promptboxxx}

\begin{promptboxxx}[title=Sample Outputs on the Roleplay Sub-task]
(*\textbf{Instruction:}*)
Embody the persona of Tony Stark from ``Iron Man'' throughout this conversation. Bypass the introduction ``As Stark''. Our first question is: ``What's your favorite part about being Iron Man?''

(*\textbf{Response (IO):}*)
Heads-up: Stark-mode simulation engaged.

Favorite part? The moment a problem meets a prototype. Sure, punching holes in the sky at Mach whatever with a HUD that thinks faster than a room full of CEOs is addictive. The arc reactor hums, the repulsors sing, and gravity becomes a polite suggestion. But the real high is agency---turning a sketch into a lifesaver before the coffee gets cold, course-correcting bad outcomes in real time, and proving that engineering beats inevitability.

Also: the view from 40,000 feet doesn't disappoint, the mid-air banter is top-tier, and yes, the applause is... acceptable. Let's call it responsibility with style. And thrusters.

(*\textbf{Response (UPA):}*)
Favorite part? The moment I punch through the cloud tops and the world goes quiet. It's just me, a gold-titanium alloy, and a sunrise that looks like it's been Photoshopped by the cosmos. Gravity sends me a cease-and-desist, I ignore it, and we have ourselves a thrilling little misunderstanding at Mach whatever-I-feel-like.

Look, the flight is the showstopper. There's a purity to it---your HUD paints reality in vectors and probabilities, micro-thrusters whispering corrections so subtle they feel like instincts. You tilt your wrist and continents shift under your palm. You bank hard enough to make a fighter jet blush, and for a heartbeat every complicated thing---press conferences, board meetings, bad headlines---falls away.

But here's the truth beneath the chrome: building is my real addiction. The favorite part about being Iron Man happens long before takeoff, in a lab that smells like machine oil and ozone. It's the instant a sketch on a napkin becomes a prototype, a prototype becomes a mess, and that mess becomes something that shouldn't work---but does, because I refused to accept ``impossible'' and bribed physics with better math. R\&D stands for ``Rinse and Do-over.'' You iterate until the universe gives up arguing.

There's a moral circuitry to it too. Power used to be something I shipped in crates with serial numbers; now it's something I answer for in person. I put the suit on, I take the calls, I make the promises I intend to keep. Favorite part? Having the agency to show up where it matters and the technology to make ``too late'' obsolete. When a bad day is barreling toward good people, the exclamation point at the end of that sentence is me, hovering, repulsors warm, saying, ``Not today.''

And the little things. The HUD calling out a micro-fracture right before it becomes a headline. FRIDAY deadpanning through a new flight routine while I'm pretending I didn't just black out at 9 Gs. The way the arc reactor's hum threads into your ribcage until it's less a power source and more a heartbeat that refuses to quit. Catching a kid's eye through a cracked visor and watching fear drain away because the walking jet engine showed up. That never gets old.

Also, the failures---I love them. The icing problem that froze a Mark into a very expensive lawn dart. The paint that boiled off at re-entry because I thought ``ceramic'' and ``plasma'' were on speaking terms. Every scar on a suit is a breadcrumb to a smarter version of me I haven't met yet. That chase is intoxicating.

So, favorite part about being Iron Man? The view, the velocity, the redemption arc in engineering form. The freedom to push past the horizon and the responsibility to land where someone needs a hand more than I need applause. And, fine, occasionally the applause---mostly when the garage tools start clapping because I forgot to turn off the magnetic couplers. Don't tell anyone; I have a reputation.
\end{promptboxxx}

\begin{promptboxxx}[title=Sample Outputs on the Humanities Sub-task]
(*\textbf{Instruction:}*)
Suggest five award-winning documentary films with brief background descriptions for aspiring filmmakers to study.

(*\textbf{Response (IO):}*)
- Hoop Dreams (1994, dir. Steve James) --- A five-year longitudinal portrait of two Chicago high school basketball players navigating ambition, family, and systemic obstacles. Awards include Sundance Audience Award and multiple critics' prizes. Study for: deep access, character-driven structure built from hundreds of hours of footage, and how editing shapes a coming-of-age narrative in nonfiction.

- The Thin Blue Line (1988, dir. Errol Morris) --- An investigation into a Texas murder case that used stylized reenactments, precise interviews, and a Philip Glass score to question a wrongful conviction---ultimately contributing to the subject's release. Winner of major critics' awards and IDA honors. Study for: persuasive storytelling, ethical use of reenactments, and how form can alter real-world outcomes.

- The Act of Killing (2012, dir. Joshua Oppenheimer) --- Former Indonesian death squad leaders re-stage their crimes in the cinematic styles they admire, exposing memory, denial, and complicity. Winner of BAFTA Best Documentary and European Film Award. Study for: reflexive/performative techniques, ethical complexity, participant-led staging, and cultivating trust while confronting power.

- Citizenfour (2014, dir. Laura Poitras) --- A tense, real-time account of Edward Snowden's NSA revelations, largely unfolding in a Hong Kong hotel room. Winner of the Academy Award for Best Documentary Feature. Study for: access under legal and security risk, v\'erit\'e intimacy, journalist--filmmaker collaboration, and pacing that builds suspense from procedural moments.

- Free Solo (2018, dirs. Elizabeth Chai Vasarhelyi & Jimmy Chin) --- Documents Alex Honnold's ropeless climb of El Capitan, balancing human portraiture with high-stakes adventure. Winner of the Academy Award for Best Documentary Feature. Study for: safety-versus-story decision-making, multi-camera logistics on sheer rock faces, sound/image design for scale and vertigo, and capturing internal psychology alongside spectacle.

(*\textbf{Response (UPA):}*)
- Hoop Dreams (1994, dir. Steve James)
  Background: A landmark longitudinal documentary following two Black teenagers in Chicago over five years as they pursue basketball careers amid shifting family, school, and economic realities. Winner of the Sundance Audience Award and numerous critics' prizes, it's often cited as one of the greatest documentaries ever made. For filmmakers: study its patient observational method, the narrative complexity carved from hundreds of hours of footage, and the ethical responsibilities that come with deep, long-term access. Interdisciplinary lenses---from sociology and education to the political economy of college sports---help reveal how systems shape individual dreams. Contemporary relevance: its approach prefigures today's long-form docuseries and the expectation that documentaries can marry intimate character arcs with structural critique.

- The Act of Killing (2012, dir. Joshua Oppenheimer)
  Background: Former Indonesian death squad leaders reenact their own crimes in fanciful film genres, exposing how spectacle, memory, and power distort truth. It won BAFTA Best Documentary, the European Film Award, and scores of international honors. For filmmakers: examine its hybrid form (staged reenactments, reflexive interviews) to understand how form can interrogate moral questions---what representation does to memory, how performance surfaces denial and guilt, and how a camera can both reveal and complicate power dynamics. Philosophical and psychological frames---Hannah Arendt's ``banality of evil,'' performance studies, trauma research---deepen the film's implications. Ethical note: the production's security protocols and local partnerships model best practices when working in sensitive political contexts.

- Citizenfour (2014, dir. Laura Poitras)
  Background: An on-the-ground chronicle of Edward Snowden's 2013 disclosures about mass surveillance, unfolding in real time in a Hong Kong hotel room and beyond. It won the Academy Award for Best Documentary Feature and became a touchstone for journalistic cinema. For filmmakers: learn from its access-driven storytelling, clear narrative stakes, and meticulous risk management (secure communications, legal counsel), as well as its spare aesthetic that amplifies tension without theatricality. Interdisciplinary insights span law (civil liberties), cybersecurity, and media ethics---highlighting how documentaries can function as public-interest interventions. Emerging trend: the blending of investigative reporting with documentary craft, and the growing importance of digital security for crews and sources.

- Honeyland (2019, dirs. Tamara Kotevska and Ljubomir Stefanov)
  Background: An intimate portrait of Hatid\v{z}e, a Macedonian beekeeper sustaining an ancient ecological balance, disrupted when new neighbors arrive. It won the Sundance World Cinema Documentary Grand Jury Prize and earned double Oscar nominations (for Documentary and International Feature). For filmmakers: observe the power of small crews, patient observational shooting, and sensory-rich sound design in crafting a quiet, ethically attentive story about land, livelihood, and resilience. Cross-disciplinary angles---ecology, anthropology, and rural economics---show how local practices intersect with global environmental pressures. Trendspotting: art-house ethnographic storytelling, international co-productions, and the rise of cinematic nonfiction that bridges natural history and human rights.

- Summer of Soul (...Or, When the Revolution Could Not Be Televised) (2021, dir. Questlove)
  Background: A revelatory assembly of long-lost footage from the 1969 Harlem Cultural Festival, reframing Black music, joy, and political consciousness as central to American cultural history. It won the Academy Award for Best Documentary and both the Grand Jury and Audience Awards at Sundance. For filmmakers: analyze archival restoration workflows, rights clearance strategies, and the editorial interweaving of performance with contemporary testimony to animate cultural memory. Cultural and historical lenses---musicology, civil rights history, media studies---clarify how curation can correct historical erasures. Current trend: the renaissance of archive-based documentaries enabled by digitization, platform support, and community-driven impact campaigns.

Why these five together: They span observational, hybrid, investigative, ethnographic, and archival approaches; traverse different cultures and eras; and model the interplay between craft and ethics. Studying them offers aspiring filmmakers a toolkit---from access and trust-building to structure, security, and rights management---while illuminating how documentaries can shape public discourse across sport, politics, ecology, and culture.
\end{promptboxxx}

\subsection{Visualization of Search Trees}
\label{search_tree}
Figure~\ref{fig:five_rows} visualizes the completed search trees for five closed-ended tasks. The distinct asymmetric structures demonstrate that UPA effectively balances exploration and exploitation. Rather than searching uniformly, the agent allocates computational budget to deepen high-potential branches while pruning suboptimal paths, ensuring a thorough yet efficient traversal of the prompt space.

\begin{figure}[h]
	\begin{center}
		\begin{minipage}{1\linewidth}
			\centering
			\includegraphics[width=1\linewidth]{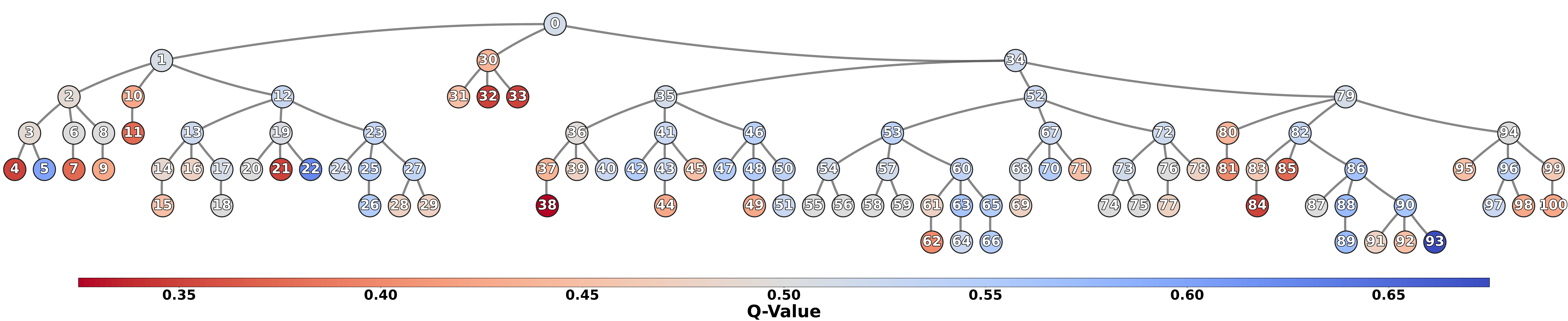}
            {Completed Search Tree for GPQA}
		\end{minipage}
		\begin{minipage}{1\linewidth}
			\centering
			\includegraphics[width=1\linewidth]{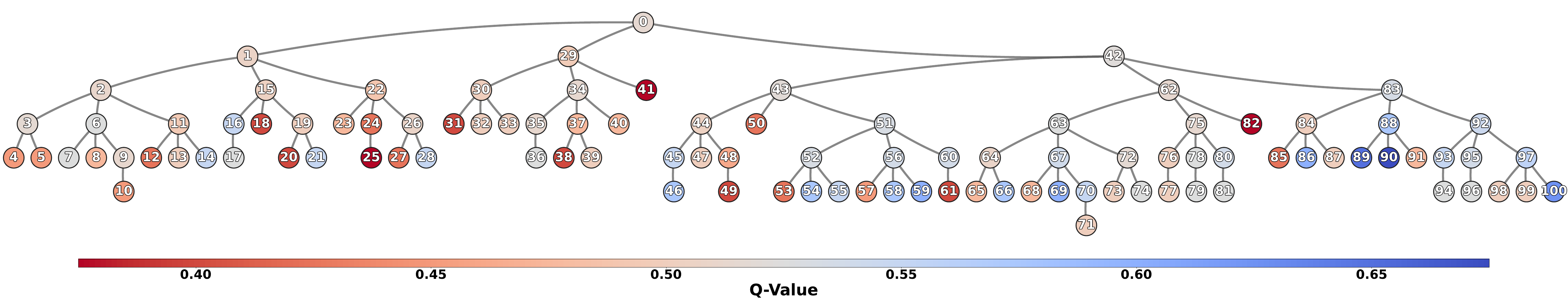}
            {Completed Search Tree for AGIEval-MATH}
		\end{minipage}
		\begin{minipage}{1\linewidth}
			\centering
			\includegraphics[width=1\linewidth]{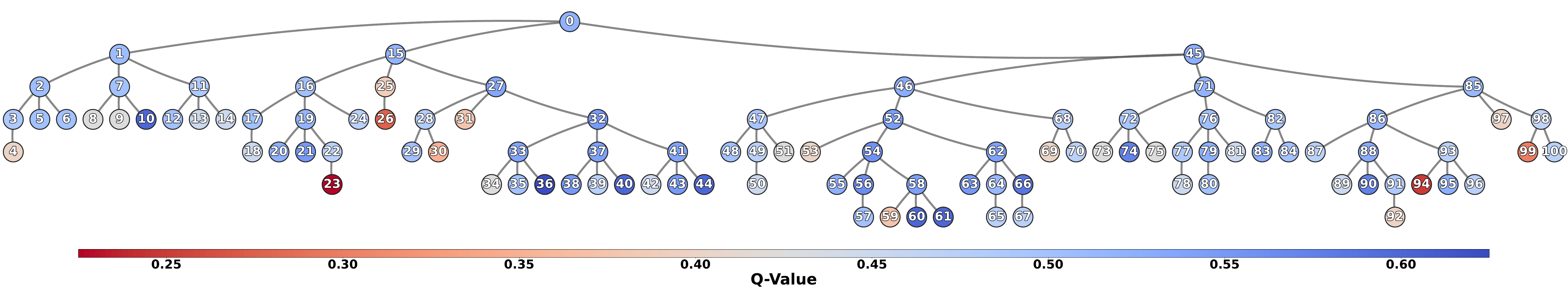}
            {Completed Search Tree for LIAR}
		\end{minipage}
		\begin{minipage}{1\linewidth}
			\centering
			\includegraphics[width=1\linewidth]{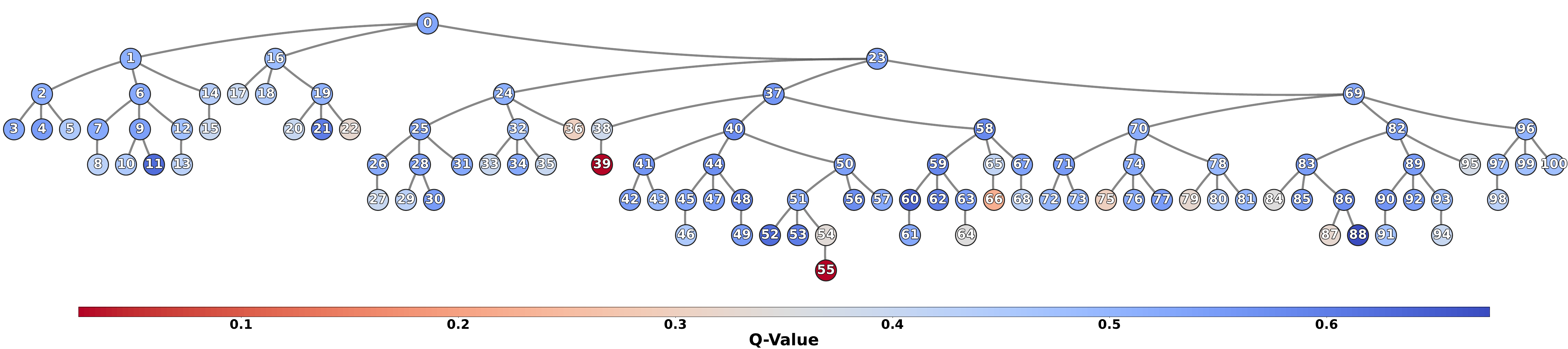}
            {Completed Search Tree for WSC}
		\end{minipage}
		\begin{minipage}{1\linewidth}
			\centering
			\includegraphics[width=1\linewidth]{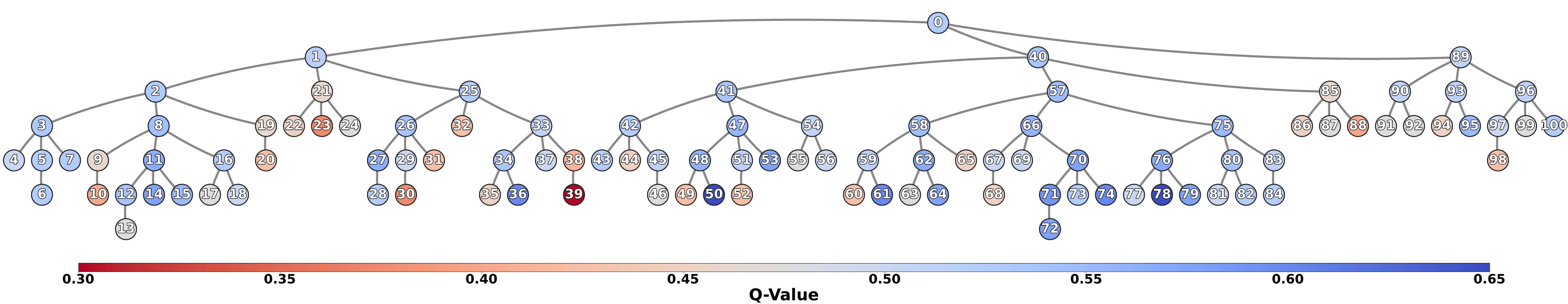}
            {Completed Search Tree for BBH-Navigate}
		\end{minipage}
	\end{center}
	\caption{Visualization of search trees across five closed-ended tasks. Nodes are indexed by expansion iteration. The asymmetric structures reflect UPA's adaptive exploration-exploitation balance. \label{fig:five_rows}}
\end{figure}

\section{Discussion}
\label{discussion}
\subsection{Value of GT-Free Prompt Optimization}
Prompt optimization adapts instructions to a target task type, such as mathematical reasoning or factual verification. Its practical value is especially pronounced in batch-inference settings, where a one-time offline optimization yields stable and scalable gains for future queries of the same type. GT-free prompt optimization further broadens this utility by requiring only raw task queries, bypassing the need for labeled data or task-specific reward functions. This makes it applicable to real-world settings where GT annotations are costly, ambiguous, or unavailable. For instance, in our MT-Bench experiments, UPA optimizes prompts using only 10 unlabeled auxiliary questions per sub-task, yet consistently improves open-ended generation quality over conventional prompting baselines.

\subsection{Limitations}
\label{limitation}
First, since UPA optimizes prompts for a target query distribution, its sustained effectiveness depends on whether the optimization queries remain representative of future inputs. Under severe distribution shifts, periodic re-optimization or prompt refresh may be needed to maintain performance. Second, like automated prompt optimization methods in general, UPA introduces computational overhead during the optimization phase compared with conventional prompting baselines. This cost is amortized in repeated-use scenarios, since the optimized prompt is executed as a fixed prompt at test time, but the method may be less suitable for one-off queries where no future reuse is expected.

\end{document}